\documentclass[10pt]{article} 
\usepackage[preprint]{tmlr}
\usepackage[T1]{fontenc}   
\usepackage{microtype}     

\usepackage{amsmath,amsfonts,bm}

\def\eqref#1{equation~\ref{#1}}

\def\1{\bm{1}}

\def\rh{{\textnormal{h}}}

\def\rw{{\textnormal{w}}}

\DeclareMathAlphabet{\mathsfit}{\encodingdefault}{\sfdefault}{m}{sl}
\SetMathAlphabet{\mathsfit}{bold}{\encodingdefault}{\sfdefault}{bx}{n}

\usepackage{pgfplots}
\pgfplotsset{compat=1.18}
\usepgfplotslibrary{groupplots}
\usepackage{hyperref}
\usepackage{url}

\usepackage{tabularx} 
\definecolor{linknavy}{RGB}{20,60,120}
\definecolor{citemaroon}{RGB}{120,30,60}
\hypersetup{
  colorlinks=true,
  linkcolor=linknavy,
  citecolor=citemaroon,
  urlcolor=linknavy,
  filecolor=linknavy,
  breaklinks=true,
}

\usepackage{booktabs}
\usepackage{amsfonts}
\usepackage{amsmath}
\usepackage{amssymb}
\usepackage{nicefrac}
\usepackage{microtype}

\usepackage{graphicx}
\usepackage{natbib}
\usepackage{caption}
\usepackage{subcaption}
\usepackage{float}
\usepackage{xcolor}
\usepackage{pifont}
\newcommand{\cmark}{\ding{51}}  
\newcommand{\xmark}{\ding{55}}  
\usepackage{colortbl}
\usepackage{tikz}
\usetikzlibrary{positioning,fit,backgrounds,arrows.meta,calc}
\usepackage{array}
\usepackage{mathtools}
\usepackage[framemethod=tikz]{mdframed}
\usepackage{enumitem}
\usepackage{amsthm}
\usepackage{placeins}
\usepackage{multirow}
\usepackage[table]{xcolor}   
\usepackage{booktabs}
\usepackage{array}           
\usepackage[most]{tcolorbox}
\tcbuselibrary{skins, breakable}

\usepackage{pgfplots}
\pgfplotsset{compat=1.18}
\usepackage{siunitx}  
\usepackage{algorithm}
\usepackage{algpseudocode}
\usepackage{listings}
\usepackage{nameref}
\usepackage{xurl}
\usepackage{algpseudocodex}   
\floatstyle{ruled}\restylefloat{algorithm}
\algrenewcommand{\algorithmiccomment}[1]{\hfill\ttfamily\scriptsize\textcolor{black!60}{//~#1}}
\algrenewcommand{\algorithmicrequire}{\textbf{Input:}}
\algrenewcommand{\algorithmicensure}{\textbf{Output:}} 
\definecolor{boxheadgray}{gray}{0.88}
\definecolor{arrowgray}{RGB}{100,100,100}
\definecolor{promptbg}{RGB}{245,248,252}
\definecolor{outputbg}{RGB}{245,252,246}
\definecolor{srcolor}{RGB}{80,100,160}
\definecolor{supportedgreen}{RGB}{30,130,60}

\definecolor{learnablebadge}{RGB}{41,98,171}
\definecolor{fixedbadge}{RGB}{178,58,58}

\definecolor{C1bg}{RGB}{232,244,253}
\definecolor{C1fr}{RGB}{25,100,160}
\definecolor{C2bg}{RGB}{243,232,255}
\definecolor{C2fr}{RGB}{100,40,170}
\definecolor{C3bg}{RGB}{232,245,233}
\definecolor{C3fr}{RGB}{30,120,50}
\definecolor{C4bg}{RGB}{255,243,224}
\definecolor{C4fr}{RGB}{180,80,10}
\definecolor{passgreen}{RGB}{27,94,32}

\definecolor{thmfill}{RGB}{242,244,251}
\definecolor{thmline}{RGB}{78,101,168}
\mdfdefinestyle{thmstyle}{%
  linewidth=0.7pt, linecolor=thmline!75, backgroundcolor=thmfill,
  roundcorner=3pt, nobreak=true,                 
  innertopmargin=4pt, innerbottommargin=5pt,     
  innerleftmargin=8pt, innerrightmargin=8pt,
  skipabove=5pt, skipbelow=5pt,                   
}
\theoremstyle{plain}
\newmdtheoremenv[style=thmstyle]{theorem}{Theorem}
\newmdtheoremenv[style=thmstyle]{lemma}{Lemma}
\newmdtheoremenv[style=thmstyle]{corollary}{Corollary}
\newmdtheoremenv[style=thmstyle]{proposition}[theorem]{Proposition}
\theoremstyle{definition}

\theoremstyle{definition}
\newmdtheoremenv[style=thmstyle]{definition}{Definition}   
\newmdtheoremenv[style=thmstyle]{assumption}{Assumption}   
\newtheorem{remark}{Remark}     
\newtcolorbox{promptbox}[2]{
  colback=promptbg, colframe=boxheadgray,
  title=#1, fonttitle=\small\sffamily\bfseries, coltitle=white,
  boxrule=0.8pt, arc=3pt, breakable, enhanced,
  left=6pt, right=6pt, top=4pt, bottom=4pt,
  attach boxed title to top left={xshift=6pt, yshift=-2mm},
  boxed title style={colback=#2, arc=2pt, boxrule=0pt},
}

\tcbset{
  pipe/.style={
    enhanced, arc=4pt, boxrule=0.7pt,
    left=10pt, right=10pt, top=9pt, bottom=9pt,   
    width=\linewidth, halign=left,
    fontupper=\small,                              
    fonttitle=\small\bfseries,
    attach boxed title to top left={yshift=-2.6mm, xshift=9pt},
    boxed title style={arc=2pt, boxrule=0.5pt, left=5pt, right=5pt, top=2pt, bottom=2pt}
  }
}
\newcommand{\pipearrow}[1]{\vspace{2pt}{\centering\color{arrowgray}\itshape $\downarrow$\enspace #1\par}\vspace{2pt}} 
\title{Prompt Codebooks: Discrete Compositional Optimization\\
for Language Model Instruction Refinement}

\author{%
  \name Jyotirmoy Nath \email jyotirmoy.nath@ee.iitd.ac.in \\
  \addr Indian Institute of Technology Delhi, India
  \AND
  \name Neeraj Kumar \email neerajkr2k14@gmail.com \\
  \addr Indian Institute of Technology Delhi, India
  \AND
  \name Brejesh Lall \email brejesh@ee.iitd.ac.in \\
  \addr Indian Institute of Technology Delhi, India
}

\def\month{MM}  
\def\year{YYYY} 
\def\openreview{\url{https://openreview.net/forum?id=XXXX}} 

\begin{document}

\maketitle

\begin{abstract}

Automatic prompt optimization (APO)~\citep{pryzant2023protegi, yuksekgonul2025textgrad, khattab2024dspy} has driven significant gains in LLM-based agentic workflows~\citep{yao2023react, shinn2023reflexion}. However, most existing methods treat each task's prompt as a monolithic, instance-blind string optimized through global edits, producing brittle updates and preventing the reuse of learned sub-behaviors. We propose \textbf{Prompt Codebook Optimization (PCO)}, a novel compositional prompt optimization framework that recasts APO as discrete learning over a finite vocabulary of natural-language \emph{instincts}---atomic, reusable instruction units. PCO organizes prompt-construction knowledge in a discrete codebook: an LLM-based encoder routes each input to a small subset of entries, and a generator composes them into a prompt for the executor. A critic emits a structured verdict that an attribution operator decomposes into per-variable textual gradients, and the encoder, generator, critic, and codebook are trained jointly under a language-valued min--max objective. The resulting routing is \emph{per-instance}: different inputs in the same task receive different instinct compositions. Across six benchmarks, PCO improves aggregate performance over zero-shot by $+13.91$ points on LLaMA-3.1-8B and $+13.50$ points on Qwen3-8B. PCO also outperforms the strongest baseline in aggregate by $+4.93$ points (LLaMA-3.1-8B) and $+3.89$ points (Qwen3-8B), while simultaneously reducing deployed prompt length by up to $14.1\times$ vs.\ MIPROv2 and $3.0\times$ vs.\ GEPA on Qwen3-8B.
\end{abstract}

\begin{figure}[t]
\centering
\definecolor{pcogreen}{RGB}{27,120,55}\definecolor{gepaorange}{RGB}{221,132,82}
\definecolor{gepamerge}{RGB}{153,79,0}\definecolor{miprogray}{RGB}{140,140,140}
\definecolor{frontier}{RGB}{120,180,120}\definecolor{note}{RGB}{170,60,60}
\resizebox{\columnwidth}{!}{%
\begin{tikzpicture}
\begin{groupplot}[
  group style={group size=2 by 1, horizontal sep=1.5cm},
  width=7.4cm, height=5.7cm,
  xmode=log, log basis x=10, xmin=520, xmax=9500, ymin=34, ymax=62,
  xtick={700,1000,2000,5000,8000}, xticklabels={0.7k,1k,2k,5k,8k},
  ytick={35,40,45,50,55,60},
  xlabel={Deployed prompt length (tokens, log)}, ylabel={Aggregate score (\%)},
  xlabel style={font=\footnotesize}, ylabel style={font=\footnotesize},
  tick label style={font=\scriptsize},
  grid=both, major grid style={black!12}, minor grid style={black!5},
  axis line style={black!60}, tick style={black!50}, clip=false,
]

\nextgroupplot
\fill[frontier!22] (axis cs:520,54.84) rectangle (axis cs:1409,62);

\addplot[dashed, black!45, line width=0.7pt] coordinates {(520,45.23)(9500,45.23)};
\node[font=\scriptsize, text=black!50, anchor=north east] at (axis cs:9300,45.0){zero-shot 45.2};
\addplot[only marks, mark=*, miprogray, mark size=3pt] coordinates {(6439,47.84)};
\node[font=\scriptsize, anchor=east, text=miprogray] at (axis cs:6000,48.4){MIPROv2\,};
\addplot[only marks, mark=square*, gepamerge, mark size=2.8pt] coordinates {(1553,52.41)};
\node[font=\scriptsize, anchor=north west, text=gepamerge] at (axis cs:1630,52.0){GEPA+Merge};
\addplot[only marks, mark=triangle*, gepaorange, mark size=3.6pt] coordinates {(1409,54.84)};
\node[font=\scriptsize, anchor=south west, text=gepaorange!85!black] at (axis cs:1480,55.2){GEPA};
\addplot[only marks, mark=star, pcogreen, mark size=6pt] coordinates {(709,58.73)};
\node[font=\scriptsize\bfseries, text=pcogreen, anchor=south] at (axis cs:709,59.4){PCO};

\nextgroupplot[legend style={at={(0.97,0.96)}, anchor=north east, font=\scriptsize,
    draw=black!40, fill=white, fill opacity=0.92, text opacity=1, cells={anchor=west},
    row sep=0.5pt, inner sep=2.5pt}]
\fill[frontier!22] (axis cs:520,44.51) rectangle (axis cs:1273,62);
\addplot[dashed, black!45, line width=0.7pt] coordinates {(520,35.54)(9500,35.54)};
\addlegendentry{zero-shot}
\node[font=\scriptsize, text=black!50, anchor=south] at (axis cs:2400,35.8){zero-shot 35.5};
\addplot[only marks, mark=*, miprogray, mark size=3pt] coordinates {(5765,38.62)};
\addlegendentry{MIPROv2}
\node[font=\scriptsize, anchor=east, text=miprogray] at (axis cs:5400,38.62){MIPROv2\,};
\addplot[only marks, mark=square*, gepamerge, mark size=2.8pt] coordinates {(1356,44.37)};
\addlegendentry{GEPA+Merge}
\addplot[only marks, mark=triangle*, gepaorange, mark size=3.6pt] coordinates {(1273,44.51)};
\addlegendentry{GEPA}
\node[font=\scriptsize, anchor=north west, text=gepaorange!85!black] at (axis cs:1330,43.6){GEPA / +Merge};
\addplot[only marks, mark=star, pcogreen, mark size=6pt] coordinates {(649,49.44)};
\addlegendentry{PCO (ours)}
\node[font=\scriptsize\bfseries, text=pcogreen, anchor=south] at (axis cs:649,50.11){PCO};
\end{groupplot}

\node[font=\bfseries\small, anchor=south] at ([yshift=2pt]group c1r1.north) {Qwen3-8B};
\node[font=\bfseries\small, anchor=south] at ([yshift=2pt]group c2r1.north) {LLaMA-3.1-8B};
\node[font=\small] at ([yshift=-0.95cm]group c1r1.south) {(a)};
\node[font=\small] at ([yshift=-0.95cm]group c2r1.south) {(b)};
\end{tikzpicture}}
\caption{A comparison of PCO against state-of-the-art prompt optimizers on aggregate score (six benchmarks) versus deployed prompt length for \textbf{(a)} Qwen3-8B and \textbf{(b)} LLaMA-3.1-8B. PCO attains the highest score with the shortest prompt on both models, using $9.1\times$ (Qwen3-8B) and $8.9\times$ (LLaMA-3.1-8B) fewer tokens than MIPROv2 on aggregate.}
\label{fig:pareto}
\end{figure}   
\section{Introduction} Large language models (LLMs) increasingly operate inside \textbf{compositional agentic workflows}, pipelines where one agent plans, another invokes tools, a third verifies, and feedback loops refine all three~\citep{yao2023react,shinn2023reflexion,wu2023autogen,khattab2024dspy}. These systems are non-differentiable by construction: tool calls, retrieval, and discrete planning steps preclude analytic gradients. Within this landscape, \textbf{automatic prompt optimization (APO)} is a core sub-problem: every agent is a composition of LLM calls conditioned on prompts that determine the planner's decisions, the verifier's calibration, and the safety of the system as a whole. Most prior APO methods, including black-box search~\citep{zhou2022ape,shin2020autoprompt},
RL over tokens~\citep{deng2022rlprompt},
evolutionary mutation~\citep{guo2024evoprompt,agrawal2025gepa},
reflective refinement~\citep{madaan2023selfrefine},
and textual-gradient optimization of prompt strings~\citep{pryzant2023protegi,yuksekgonul2025textgrad}, treat the prompt for each task as a \textbf{monolithic, instance-blind text object}: a single string applied identically
to every input and optimized through global edits. Consequently, the formulation is both
\textbf{instance-blind}, because one prompt serves all inputs, and \textbf{globally entangled}, because a critique
intended to fix one failure may rewrite the entire prompt and disrupt other well-tuned behaviors.

Meanwhile, a parallel decade of generative modeling has shown that \textbf{discrete latent codebooks}~\citep{vandenoord2017vqvae,razavi2019vqvae2,chang2022maskgit}, paired with a \textbf{generator--critic} training signal~\citep{goodfellow2014gan,arjovsky2017wgan,esser2021vqgan}, yield state-of-the-art results in image synthesis, neural audio coding~\citep{defossez2022encodec}, and token-based multimodal modeling. A defining property of such systems is that they are trained \textbf{per distribution}: their power is intra-domain, with codes specialized to a single data regime. APO has exactly this shape: each task defines its own input distribution, reward, and optimal prompting policy, and is in practice run per-task---yet no prior APO method organizes prompts as compositions over a discrete codebook of natural-language instincts. We close this gap with \textbf{Prompt Codebook Optimization (PCO)}, a framework that introduces a discrete codebook of natural-language \emph{instincts} --- atomic, reusable instructions that are themselves optimizable variables --- as a first-class object in prompt optimization. A \textbf{prompt encoder} (an LLM) routes each input $x$ to a small subset of instincts via \emph{semantic routing}, allowing the encoder to select instincts based on meaning rather than vector proximity. A \textbf{prompt generator} composes the selected instincts, conditioned on $x$, into a fluent prompt dispatched to the executor. A \textbf{critic} emits a structured natural-language verdict that an attribution operator partitions into per-variable textual gradients~\citep{yuksekgonul2025textgrad}, propagating \textbf{component-wise} through the generator, the active instincts, the encoder's routing policy, and the critic's own criteria. The full system is trained end-to-end under a language-valued min--max objective.

This architecture confers four advantages over monolithic per-task APO, each corresponding to a capability useful for broader agentic workflows. \textbf{(i) Per-instance adaptive prompting}: the encoder routes different inputs \emph{within the same task} to different instinct compositions --- hard instances invoke verification-heavy instincts, easier inputs use lightweight ones --- a regime structurally absent from monolithic optimizers. \textbf{(ii) Dense, behavior-level supervision}: the critic~\citep{arjovsky2017wgan} emits a structured natural-language verdict rather than a sparse binary reward, exposing not just \emph{whether} a prompt failed but \emph{which behavior} was responsible, providing a richer training signal than scalar feedback alone. \textbf{(iii) Regularization through a discrete bottleneck}: by construction, the finite codebook constrains the policy to compose prompts from a small shared inventory rather than editing an unbounded string, biasing it toward instance-general structure. \textbf{(iv) Localized credit assignment}: attribution routes feedback to the specific instinct responsible for a failure, leaving unrelated instincts untouched --- unlike monolithic prompt optimization, where any update may rewrite the entire string, and unlike fixed-template modular prompting, PCO allocates credit dynamically across a routed codebook.

\begin{figure}[t]
\centering
\resizebox{\columnwidth}{!}{%
\begin{tikzpicture}[
    >=Stealth,
    font=\sffamily\small,
    base_box/.style={draw=black, thick, rounded corners=3pt, align=center,
                      minimum height=0.85cm, inner sep=3pt, fill=gray!5},
    blue_box/.style={base_box, fill=blue!7},
    red_box/.style={base_box, fill=red!6, draw=red!55!black},
    cb_box/.style={draw=black!70, thick, rounded corners=2pt, align=center,
                    minimum height=0.85cm, inner sep=3pt, fill=blue!4},
    tiny_sq/.style={draw=black, thin, rounded corners=1pt, minimum size=0.2cm, inner sep=0pt},
    token/.style={circle, draw=black, thick, minimum size=0.26cm, inner sep=0pt},
    divider/.style={draw=black!40, dashed, thick},
]

\begin{scope}[xshift=-0.3cm] 
\node[anchor=west, font=\bfseries] (labelA) at (-1.0, 2.4) {(a) Monolithic APO};
\node[anchor=west, font=\bfseries] (labelB) at (7.5, 2.4) {(b) Prompt codebook (Ours)};

\node[base_box, minimum width=1.0cm] (x_mono) at (-2, 0) {Input\\$x$};
\node[red_box, minimum width=2.1cm] (theta) at (0, 0) {Monolithic\\Prompt};
\node[base_box, minimum width=1.2cm] (llm_mono) at (2.5, 0) {Executor};
\node[text=red!70!black, font=\bfseries, align=center] (y_mono) at (4.9, 0) {Answer\\$y$};

\draw[->, thick] (x_mono) -- (theta);
\draw[->, thick] (theta) -- (llm_mono);
\draw[->, thick] (llm_mono) -- (y_mono);

\node[token, fill=orange!70] (a_t1) at (-2, -0.8) {};
\node[token, fill=teal!70]   (a_t2) at (-2, -1.25) {};
\node[font=\tiny, anchor=east] at (-2.1,-1.175) {$x_1,x_2$};

\draw[->, thick, orange!80!black]
    (a_t1.east) .. controls (0.6,-1.0) and (1.1,-0.4) .. (theta.south);
\draw[->, thick, teal!80!black]
    (a_t2.east) .. controls (0.6,-1.45) and (1.3,-0.5) .. (theta.south);

\node[font=\tiny, align=center, text width=2.6cm] at (0.0,-1.85)
      {\textit{same task, different input$\rightarrow$ same prompt}};
\end{scope}
\draw[divider] (5.15, 2.3) -- (5.15, -1.6);

\node[base_box, minimum width=1.0cm] (x_pco) at (6.1, 0) {Input\\$x$};
\node[blue_box, minimum width=1.35cm] (enc) at (8.0, 0) {Encoder\\$\theta$};
\node[blue_box, minimum width=1.35cm] (gen) at (10.3, 0) {Generator\\$\phi$};
\node[base_box, minimum width=1.2cm] (llm_pco) at (12.2, 0) {Executor};
\node[text=green!55!black, font=\bfseries, align=center] (y_pco) at (14.5, 0) {Answer\\$y$};

\node[cb_box, minimum width=2.4cm] (cb) at (9.15, 1.4) {\scriptsize\textbf{Codebook $C$}};
\foreach \i/\c in {1/cyan!40, 2/green!40, 3/yellow!40, 4/purple!40, 5/red!40} {
  \pgfmathsetmacro{\xoff}{(\i-3)*0.32}
  \node[tiny_sq, fill=\c] at ([xshift=\xoff cm, yshift=-0.24cm]cb.center) {};
}

 \node[
    cb_box,
    minimum width=2.4cm,
    minimum height=1cm,
    align=center
] (cb) at (9.15,1.45) {};

\node[font=\scriptsize\bfseries] at ($(cb.center)+(0,0.36)$)
    {Codebook $C$};

\node[font=\tiny] at ($(cb.center)+(0,0.08)$)
    {Textual Instincts};

\foreach \x/\c in {
    -1/cyan!45,
    -0.7/green!45,
     -0.4/yellow!60,
     -0.1/purple!45,
     0.2/red!45}
{
    \node[
        tiny_sq,
        fill=\c,
        minimum size=4pt
    ] at ($(cb.center)+(\x,-0.34)$) {};
}

\node[font=\scriptsize] at ($(cb.center)+(0.5,-0.34)$) {$\cdots$}; 
\draw[->, thick] (x_pco) -- (enc);
\draw[->, thick] (enc) -- (gen);
\draw[->, thick] (gen) -- (llm_pco);
\draw[->, thick, rounded corners=3pt]
    (x_pco.south) -- ++(0,-0.20) -| (gen.south);
\draw[->, thick] (llm_pco) -- (y_pco);

\draw[->, thick] (enc.north) -- ++(0, 0.35) -| ($(cb.south) + (-0.55,0)$);
\draw[->, thick] ($(cb.south) + (0.55,0)$) -- ++(0,-0.35) -| (gen.north);

\node[token, fill=orange!70] (b_t1) at (6.1, -1.1) {};
\node[token, fill=teal!70]   (b_t2) at (6.1, -1.45) {};
\node[font=\tiny, anchor=east] at (6.1,-1.275) {$x_1,x_2$};

\draw[->, thick, orange!80!black]
    (b_t1.east) .. controls (7.2,-1.05) and (7.7,-0.4) .. (enc.south);
\draw[->, thick, teal!80!black]
    (b_t2.east) .. controls (7.3,-1.4) and (7.85,-0.5) .. (enc.south);

\node[font=\tiny, align=center, text width=2.7cm] at (6.65,-1.85)
      {\textit{same task, different inputs $\rightarrow$ different instincts}};

\end{tikzpicture}
}
\caption{\textbf{Architectural Comparison.} \textbf{(a)} Monolithic methods apply a single, instance-blind prompt per task. \textbf{(b)} PCO routes each input through an encoder that selects instincts from a discrete codebook, and a generator composes them into an instance-specific prompt, so distinct inputs within the same task receive different instinct compositions.}
\label{fig:flat_comparison}
\end{figure}

\begin{figure*}[t]
\centering
\resizebox{\textwidth}{!}{%
\begin{tikzpicture}[
    >=Stealth,
    font=\sffamily\normalsize,
    base_box/.style={draw=black, thick, rounded corners=5pt, align=center,
                     minimum height=1.75cm, minimum width=2.5cm, inner sep=8pt, fill=gray!5},
    blue_box/.style={base_box, fill=blue!8},
    orange_box/.style={base_box, fill=orange!12},
    frozen_box/.style={base_box, fill=blue!14},
    critic_box/.style={base_box, fill=red!8, draw=red!50!black},
    cb_box/.style={draw=black, thick, rounded corners=5pt, align=center,
                   minimum height=1.55cm, inner sep=8pt, fill=gray!5},
    tiny_sq/.style={draw=black!60, thin, minimum size=0.3cm, inner sep=0pt},
    data_arr/.style={->, thick, black, line width=1.1pt},
    grad_arr/.style={->, ultra thick, red!65!black, dashed, line width=1.6pt},
    inf_box/.style={base_box, fill=gray!7},
]

\fill[blue!3]  (-0.9, -0.7) rectangle (17.7, 9.3);
\fill[gray!4]  (18.1, -0.7) rectangle (32.4, 9.3);

\node[cb_box, minimum width=7.4cm, minimum height=1.6cm] (cb) at (6.6, 8.2) {%
    \textbf{\Large Codebook $C$}\\[4pt]
    \normalsize Textual Instincts\\[4pt]
    \phantom{x}%
};
\node[tiny_sq, fill=cyan!45]   at (4.0, 7.55) {};
\node[tiny_sq, fill=green!45]  at (5.0, 7.55) {};
\node[tiny_sq, fill=yellow!55] at (6.0, 7.55) {};
\node[tiny_sq, fill=purple!40] at (7.0, 7.55) {};
\node[tiny_sq, fill=red!40]    at (8.0, 7.55) {};
\node[font=\bfseries\Large, text=black!50] at (9.0, 7.55) {\ldots};

\node[base_box, minimum width=1.4cm] (inp)  at (1,  5.8) {Input\\$x$};
\node[blue_box]   (enc)  at (3.5,  5.8) {Prompt\\Encoder\\$\theta$};
\node[orange_box] (gen)  at (6.8,  5.8) {Prompt\\Generator\\$\phi$};
\node[frozen_box] (llm)  at (10.6, 5.8) {Executor};
\node[critic_box] (crit) at (14.2, 5.8) {Critic\\$\psi$};

\draw[data_arr] (inp)  -- (enc);
\draw[data_arr] (enc)  -- (gen);
\draw[data_arr] (gen)  -- (llm);
\draw[data_arr] (llm)  -- (crit);

\draw[data_arr] (enc.north) -| ([xshift=-0.2cm]$(cb.south west)!0.11!(cb.south east)$);
\draw[data_arr] (gen.north |- cb.south) -- (gen.north);
\draw[data_arr, rounded corners=4pt] (inp.south) -- (1.0, 4.6) -| (gen.south);

\node[draw=red!65!black, thick, dashed, rounded corners=6pt,
      fill=yellow!5, align=center, inner sep=8pt, text width=9.0cm,
      font=\normalsize]
    (tgd) at (7, 1.30) {%
    {\color{red!65!black}\Large\textbf{Textual Gradient Descent (TGD)}}\\[3pt]%
    \textit{$\ell_{\mathrm{text}}$: ``step-by-step reasoning missed;}\\%
    \textit{instincts too general for arithmetic decomposition''}
};

\draw[grad_arr, <->, ultra thick, red!65!black, line width=1.6pt, rounded corners=2pt]
    (crit.south) -- ++(0, -3.65cm) -| ($(tgd.north east)!0.53!(tgd.south east)$);

\draw[grad_arr] (enc.south |- tgd.north west) -- (enc.south);
\node[draw=black!30, rounded corners=2pt, fill=white, align=center,
      inner sep=4pt, font=\small\itshape] at (1.7, 3.75)
    {{\normalsize\boldmath\bfseries\color{red!65!black}$g_{\theta}$}\;\itshape``route to arithmetic\\decomposition instinct''};

\draw[grad_arr] ([xshift=-1cm]tgd.north) -- ($(gen.south west)!0.15!(gen.south east)$);
\node[draw=black!30, rounded corners=2pt, fill=white, align=center,
      inner sep=4pt, font=\small\itshape] at (5.8, 3.75)
    {{\normalsize\boldmath\bfseries\color{red!65!black}$g_{\phi}$}\;\itshape``compose with\\step-by-step instruction''};

\draw[grad_arr] ([xshift=-3cm]tgd.north east) -- ($(cb.south west)!0.80!(cb.south east)$);
\node[draw=black!30, rounded corners=2pt, fill=white, align=center,
      inner sep=4pt, font=\small\itshape] at (10.2, 3.75)
    {{\normalsize\boldmath\bfseries\color{red!65!black}$g_{C}$}\;\itshape``sharpen $c_3$:\\add decomposition constraint''};

\node[draw=black!30, rounded corners=2pt, fill=white, align=center,
      inner sep=4pt, font=\small\itshape] at (14.2, 3.75)
    {{\normalsize\boldmath\bfseries\color{red!65!black}$g_{\psi}$}\;\itshape``flag reasoning gap\\as primary failure''};

\node[draw=blue!45, thick, rounded corners=5pt, fill=blue!4,
      align=left, inner sep=8pt, font=\small] at (14.9, 8) {%
    {\color{blue!55!black}\textbf{\normalsize Optimizable Components}}\\[4pt]%
    $\bullet$ Prompt Encoder ($\theta$)\\%
    $\bullet$ Codebook (Textual Instincts) $C$\\%
    $\bullet$ Prompt Generator ($\phi$)\\%
    $\bullet$ Critic Evaluation Criteria ($\psi$)%
};

\draw[draw=black!30, dashed, thick] (18.0, 9.3) -- (18.0, -0.9);

\node[cb_box, minimum width=6.6cm, minimum height=1.6cm] (cb_inf) at (25.2, 8.2) {%
    \textbf{\Large Codebook $C^{*}$}\\[4pt]
    \normalsize Textual Instincts\\[6pt]
    \phantom{x}%
};
\node[tiny_sq, fill=cyan!45]   at (23.8, 7.75) {};
\node[tiny_sq, fill=green!45]  at (24.8, 7.75) {};
\node[tiny_sq, fill=yellow!55] at (25.8, 7.75) {};
\node[tiny_sq, fill=purple!40] at (26.8, 7.75) {};

\node[base_box, minimum width=1.6cm] (x_inf)   at (19.8, 5.8) {Input\\$x_{\mathrm{new}}$};
\node[inf_box]    (enc_inf) at (22.5, 5.8) {Prompt\\Encoder\\$\theta^{*}$};
\node[inf_box]    (gen_inf) at (25.2, 5.8) {Prompt\\Generator\\$\phi^{*}$};
\node[frozen_box] (llm_inf) at (28.4, 5.8) {Executor};
\node[font=\bfseries\Large, align=center, right=0.8cm of llm_inf] (ans) {Answer\\$y$};
\draw[data_arr] (x_inf)   -- (enc_inf);
\draw[data_arr] (enc_inf) -- (gen_inf);
\draw[data_arr] (gen_inf) -- (llm_inf);
\draw[data_arr] (llm_inf) -- (ans);
\draw[data_arr] (enc_inf.north) -- (enc_inf.north |- cb_inf.south);
\draw[data_arr] (gen_inf.north |- cb_inf.south) -- (gen_inf.north);
\draw[data_arr, rounded corners=4pt] (x_inf.south) -- (19.8, 4.3) -| (gen_inf.south);

\node[draw=purple!40, thick, rounded corners=5pt, fill=purple!5,
      align=center, inner sep=12pt, font=\large] at (25.2, 2.5) {%
    No critic, no gradient updates\\--- the critic and update rules are discarded.%
};

\node[font=\LARGE] at (6.5, -0.2) {\color{blue!60!black}\textbf{Training}};
\node[font=\LARGE] at (25.2, -0.2) {\color{gray!70!black}\textbf{Inference}};

\node[align=center, inner sep=8pt, font=\LARGE] at (14.0, -1.3) {%
    $\longrightarrow$ Data flow (forward)
    \qquad
    {\color{red!65!black}$\dashrightarrow$} Gradient flow (via textual feedback)
    \qquad
    \tikz[baseline=-0.5ex]{
        \node[tiny_sq, fill=cyan!45]   at (0.00,0) {};
        \node[tiny_sq, fill=green!45]  at (0.35,0) {};
        \node[tiny_sq, fill=yellow!55] at (0.70,0) {};
        \node[tiny_sq, fill=purple!40] at (1.05,0) {};
        \node[tiny_sq, fill=red!40]    at (1.40,0) {};
    }
    \, Codebook entries (textual instincts)
};
\end{tikzpicture}
}
\vspace{0.03cm}
\caption{%
Overview of PCO. The encoder $\theta$ maps input $x$ to discrete indices, selecting instincts from a
shared codebook $C$. The generator $\phi$ composes the selected instincts conditioned on $x$ into a
prompt executed by an executor LLM. A critic $\psi$ emits a natural-language verdict
$\ell_{\mathrm{text}}$, which is propagated backward via textual gradients
($g_{\theta}$, $g_{\phi}$, $g_{C}$, $g_{\psi}$) to each trainable component independently.%
}
\label{fig:pco_overview}
\end{figure*} 
 \paragraph{Contributions.}
\begingroup
\tolerance=1000
\emergencystretch=2em
\begin{enumerate}
\itemsep1pt
\item We introduce a \textbf{discrete codebook of natural-language instincts}: a finite inventory of reusable instincts selected and composed per input. Unlike query-level regeneration methods~\citep{kong2024qpo,nica2025trprompt}, PCO composes prompts from a shared codebook rather than generating from scratch; unlike continuous prompt-pool methods~\citep{wang2022l2p,jiang2024mope} and section-based decomposition~\citep{sharma2026modular}, our instincts are natural-language instructions refined via textual gradients, routed per-instance.
\item We perform \textbf{per-instance adaptive prompting} through semantic routing over the codebook, letting each input activate a different instinct subset for compositional, instance-specific prompts --- the same shared-pool routing principle as L2P~\citep{wang2022l2p} and MoPE~\citep{jiang2024mope}, realized in discrete language rather than continuous embeddings.
\item We formulate codebook, encoder, generator, and critic training as a \textbf{language-valued min--max optimization} problem with a \textbf{jointly trained critic}, extending prior textual-gradient minimax games over monolithic prompts~\citep{do2024advicl} to per-variable attribution over a routed codebook, jointly optimizing instinct content, routing, composition, and the critic's own criteria.
\item We evaluate PCO on six benchmarks spanning multi-hop reasoning, fact verification, mathematics, instruction following, and privacy-preserving delegation, achieving up to $+33.56$ points over zero-shot on \textsc{HotpotQA} (LLaMA-3.1-8B) and a $+3.89$ aggregate improvement over the strongest baseline (GEPA on Qwen3-8B).
\end{enumerate}
\endgroup

The remainder of this paper is structured as follows.
Section~\ref{sec:related} surveys related work on monolithic 
prompt optimization, discrete latent representations, and 
instance-aware prompting methods (extended in Appendix~\ref{app:related}).
Section~\ref{sec:method} formalizes the PCO framework, 
introducing the codebook, encoder, generator, and critic, 
and describes the forward pass, min--max training objective, 
component-scoped credit assignment, and inference procedure.
Section~\ref{sec:theory} establishes three theoretical 
properties of PCO: the expressivity advantage of compositional 
prompting, codebook utilization lower bounds under 
$\epsilon$-greedy exploration, and the approximate 
additivity of component-scoped credit assignment.
Section~\ref{sec:experiments} presents experimental results 
across six benchmarks and two model families, followed by 
component ablations, hyperparameter sensitivity analysis, 
critic capacity analysis, and computational efficiency; the 
attribution operator evaluation and further analyses are 
deferred to the appendix.
Section~\ref{sec:conclusion} concludes and discusses the 
scope and limitations of our analysis.

\section{Related Work}
\label{sec:related}
\noindent\textbf{Monolithic Prompt Optimization.}
Early approaches to prompt optimization rely on discrete search over candidate tokens or instructions~\citep{shin2020autoprompt,zhou2022ape}, while recent methods utilize natural-language feedback, reflection, and evolutionary mutation to iteratively refine prompts~\citep{pryzant2023protegi,madaan2023selfrefine,guo2024evoprompt,yuksekgonul2025textgrad,agrawal2025gepa}. Reinforcement learning methods such as GRPO~\citep{shao2024grpo,zuo2025}, which update model weights rather than prompts, and compositional pipeline systems like DSPy~\citep{khattab2024dspy,opsahlong2024mipro} further automate this process. However, because these standard frameworks apply global feedback to entire monolithic strings, updates remain noisy. Critically, these methods are instance-blind: learned structural components cannot be dynamically routed or reused across different inputs within the same task.

\vspace{2pt}
\noindent\textbf{Discrete and Instance-Level Representations.}
{\sloppy
Discrete latent representations model data using a finite set of reusable components~\citep{vandenoord2017vqvae,razavi2019vqvae2}. In prompt engineering, continuous prompt pools like L2P~\citep{wang2022l2p} and MoPE~\citep{jiang2024mope} use input-dependent routing over continuous expert embeddings. In contrast, PCO operates entirely in the discrete textual domain. While generative query-level models like QPO~\citep{kong2024qpo} and TRPrompt~\citep{nica2025trprompt} produce prompts from scratch, and Modular Prompt Optimization~\citep{sharma2026modular} optimizes fixed sections, PCO routes dynamically over a learned codebook of instincts, enabling structured component reuse across inputs within the same task.}


\section{Prompt Codebook Optimization (PCO)}
\label{sec:method}
\vspace{0.3em}
We propose \textbf{Prompt Codebook Optimization (PCO)}, which recasts prompt optimization as discrete compositional learning over a finite vocabulary of natural-language instincts. In place of monolithic prompt strings optimized by most prior APO methods~\citep{zhou2022ape,pryzant2023protegi,yuksekgonul2025textgrad}, PCO imposes a discrete latent bottleneck where prompts are composed from shared, reusable instincts. As shown in Figure~\ref{fig:pco_overview}, the framework routes and composes instincts using an encoder--generator--critic pipeline, jointly trained via textual gradient descent~\citep{yuksekgonul2025textgrad} under a language-valued min--max objective. PCO is language-native throughout: all parameters are strings, all gradients are textual critiques, and routing is performed semantically by an LLM rather than continuous vector quantization~\citep{vandenoord2017vqvae}.

\subsection{Problem Setup and Notation}
\label{sec:setup}
Let $\mathcal{M}$ denote a frozen large language model and $\mathcal{D} = \{(x_i, y_i^{\star})\}_{i=1}^N$ a per-task dataset with inputs $x_i \in \mathcal{X}$ and references $y_i^{\star} \in \mathcal{Y}$. Let $r: \mathcal{Y} \times \mathcal{Y} \to \mathbb{R}$ be a task-specific reward (exact match, pass-rate, or learned preference). The standard APO problem~\citep{zhou2022ape,pryzant2023protegi} seeks a single prompt $p^{\star}$ maximizing expected reward:
\begin{equation}
p^{\star} = \arg\max_{p \in \mathcal{P}} \mathbb{E}_{(x, y^{\star}) \sim \mathcal{D}} \bigl[\, r(\mathcal{M}(p, x), y^{\star}) \,\bigr].
\label{eq:apo}
\end{equation}
Equation~\ref{eq:apo} treats the prompt as a monolithic, instance-blind string. Its two structural pathologies---\emph{instance blindness} (the same $p$ for every $x$) and \emph{global entanglement} (any update may rewrite the entire string)---motivate the discrete latent decomposition introduced next.

\subsection{Discrete Latent Decomposition of Prompts}
\label{sec:decomp}

We replace the monolithic prompt with four trainable components, including a finite codebook of instincts:
\begin{itemize}\itemsep0pt

\item A \textbf{prompt encoder},
$\mathcal{E}_\theta :
\mathcal{X} \times \mathcal{C}^K \rightarrow \{1,\ldots,K\}^S$,
implemented as an LLM with system prompt $\theta$,
maps an input $x$ to $S$ discrete indices into a codebook of size $K$,
where $S \ll K$.

\item A \textbf{codebook},
$\mathcal{C} = \{c_1, \ldots, c_K\}$,
where each $c_k$ is a short natural-language directive
(e.g., ``decompose into sub-questions before answering'').

\item A \textbf{prompt generator},
$\mathcal{G}_\phi :
\mathcal{X} \times \mathcal{C}^S \rightarrow \mathcal{P}$,
implemented as an LLM with system prompt $\phi$,
composes the $S$ selected instincts, conditioned on $x$,
into a fluent prompt.

\item A \textbf{critic},
$\mathcal{D}_\psi :
\mathcal{Y} \times \mathcal{X} \times \mathcal{P} \times \mathcal{Y}
\rightarrow \mathcal{T}$,
implemented as an LLM with system prompt $\psi$,
which emits natural-language feedback in the textual-gradient space
$\mathcal{T}$.
\end{itemize}

The trainable variable set is $\Theta = \{\theta, \phi, \mathcal{C}, \psi\}$; only the target model $\mathcal{M}$ is held frozen. Under this decomposition, the monolithic prompt of equation~\ref{eq:apo} is replaced by the composed prompt
\begin{equation}
p_\Theta(x) = \mathcal{G}_\phi\bigl(x,\, \mathcal{C}[\mathcal{E}_\theta(x, \mathcal{C})]\bigr),
\label{eq:composed-prompt}
\end{equation}
so that the APO objective is now optimized over $\Theta$ through the encoder--codebook--generator pipeline.
\subsection{Forward Pass: Routing, Composition, and Execution}
\label{sec:forward}
For each input $x$, the forward pass proceeds in three steps:
\begin{align}
z_q &= \mathcal{E}_\theta(x, \mathcal{C}) \in \{1, \ldots, K\}^S, \label{eq:fwd-route} \\
p   &= \mathcal{G}_\phi(x, \{c_k : k \in z_q\}), \label{eq:fwd-gen} \\
y   &= \mathcal{M}(p, x). \label{eq:fwd-exec}
\end{align}
Equation~\ref{eq:fwd-route} is the discrete bottleneck: the encoder selects $S$ indices from a codebook of size $K$, yielding an input-dependent routing $z_q$. Equation~\ref{eq:fwd-gen} composes the selected instincts $\{c_k\}_{k \in z_q}$ into a fluent prompt $p$, conditioned on $x$ for instance-specific phrasing. Equation~\ref{eq:fwd-exec} is the sole call to the executor. Together, equations~\ref{eq:fwd-route}--\ref{eq:fwd-gen} realize per-instance adaptive prompting: distinct inputs flow through distinct instinct compositions.

\begin{algorithm}[t]
\caption{Prompt Codebook Optimization (PCO)}
\label{alg:pco}
\footnotesize
\renewcommand{\baselinestretch}{1.2}\selectfont 
\begin{algorithmic}[1]
\Require dataset $\mathcal{D}$; codebook size $K$, selection size $S$; schedule $\epsilon_0,\epsilon_{\min},\gamma$; EMA step size $\alpha$; temperature $\tau$; epochs $T$; reward $r$; frozen executor $\mathcal{M}$
\State Initialize codebook $\mathcal{C}=\{c_1,\dots,c_K\}$ and parameters $\theta,\phi,\psi$
\State $\bar r_k \gets 0~\forall k$;\quad $\epsilon \gets \epsilon_0$
\For{$t=1,\dots,T$}
  \For{$(x,y^\star)\in\mathcal{D}$}
    \If{$\mathrm{rand}()<\epsilon$}
      \State $z_q \sim \mathrm{SuccWtd}(\bar{\mathbf r},S,\tau)$ \Comment{explore: success-weighted routing, Sec.~\ref{sec:algo}}
    \Else
      \State $z_q \gets \mathcal{E}_\theta(x,\mathcal{C})$ \Comment{encoder routing, Sec.~\ref{sec:forward}}
    \EndIf
    \State $p \gets \mathcal{G}_\phi(x,\mathcal{C}[z_q])$;\quad $y \gets \mathcal{M}(p,x)$
    \State $r_t \gets r(y,y^\star)$ \Comment{binary task success}
    \State $(\ell,g_\theta,g_\phi,g_{\mathcal{C}},g_\psi) \gets \mathcal{D}_\psi(\mathcal{H})$ \Comment{critic \& attribution on trace $\mathcal{H}$, Sec.~\ref{sec:backward}}
    \State $\theta \gets \mathrm{LLM}_{\mathrm{upd}}(\theta,g_\theta)$;\ \ $\phi \gets \mathrm{LLM}_{\mathrm{upd}}(\phi,g_\phi)$ \Comment{textual-gradient step, Sec.~\ref{sec:backward}}
    \For{$k\in z_q$}
      \State $c_k \gets \mathrm{LLM}_{\mathrm{upd}}(c_k,g_{c_k})$
    \EndFor
    \State $\psi \gets \mathrm{LLM}_{\mathrm{upd}}(\psi,g_\psi)$
    \For{$k\in z_q$}
      \State $\bar r_k \gets (1-\alpha)\bar r_k + \alpha\, r_t$
    \EndFor
  \EndFor
  \State $\epsilon \gets \max(\epsilon_{\min},\gamma\epsilon)$
\EndFor
\State \Return optimized $\mathcal{C},\theta,\phi,\psi$ \Comment{deploy at inference, Sec.~\ref{sec:inference}}
\end{algorithmic}
\end{algorithm}
\subsection{Training Objective: A Critic-Regularized Discrete Bottleneck}
\label{sec:objective}
The forward pass of Section~\ref{sec:forward} produces, for each input $x$, a composed prompt $p = \mathcal{G}_\phi(x, \mathcal{C}[\mathcal{E}_\theta(x, \mathcal{C})])$ and response $y = \mathcal{M}(p, x)$. Motivated by the GAN family~\citep{goodfellow2014gan,arjovsky2017wgan}, where a critic supplies a dense continuous signal in place of binary supervision, we introduce a critic $\mathcal{D}_\psi: \mathcal{Y} \times \mathcal{X} \times \mathcal{P} \times \mathcal{Y} \to \mathcal{T}$ that plays the same functional role in the textual domain. Rather than a scalar, $\mathcal{D}_\psi$ emits a structured natural-language verdict
\begin{equation*}
\ell_{\mathrm{text}} = \mathcal{D}_\psi(y, x, p, y^{\star}) \in \mathcal{T},
\end{equation*}
identifying behaviors in $y$ deviating from $y^{\star}$, localizing failures to specific elements of $p$, and prescribing corrections. A scalarizer $\rho: \mathcal{T} \to \mathbb{R}_{\geq 0}$ projects this verdict onto a penalty --- zero when the critique is empty, larger as deviations grow more severe. PCO then trains the joint parameter set $\Theta = \{\theta, \phi, \mathcal{C}, \psi\}$ under a language-valued min--max objective:
\begin{equation}
\Theta^{\star} = \arg\max_{\theta,\phi,\mathcal{C}}\, \min_{\psi}\,
\mathbb{E}_{(x, y^{\star}) \sim \mathcal{D}} \Bigl[ r\bigl(\mathcal{M}(p, x), y^{\star}\bigr)
- \rho\bigl(\mathcal{D}_\psi(y, x, p, y^{\star})\bigr) \Bigr],
\label{eq:vqpo-minmax}
\end{equation}
where the outer player $(\theta, \phi, \mathcal{C})$ maximizes the task reward $r$ minus the critic-induced penalty $\rho$, and the inner player $\psi$ updates the critic's evaluation criteria to minimize the objective (thereby maximizing the penalty $\rho$ by discovering harder, more subtle flaws). In Algorithm~\ref{alg:pco}, equation~\ref{eq:vqpo-minmax} is optimized through the critic's textual feedback rather than a scalar penalty; the success rate $\bar{r}_k$ is an EMA of binary task success, $\bar{r}_k \gets (1-\alpha)\bar{r}_k + \alpha r_t$ with $r_t\in\{0,1\}$, so $\bar r_k\in[0,1]$. Grounding $\mathcal{D}_\psi$ on the reference $y^\star$ is designed to yield an empty critique ($\rho=0$) when $y=y^\star$, preventing degenerate critic negativity.

\subsection{Backward Pass: Component-Scoped Credit Assignment}
\label{sec:backward}

Because routing and prompt construction are discrete, PCO cannot propagate
analytic gradients through the encoder $\mathcal{E}_\theta$, codebook $\mathcal{C}$, and generator
$\mathcal{G}_\phi$. Instead, the critic $\mathcal{D}_\psi$ performs natural-language credit
assignment over the complete forward execution trace
\begin{equation}
\mathcal{H} = \big(x,\; z_q,\; \{c_k\}_{k\in z_q},\; p,\; y,\; y^\star\big),
\label{eq:trace}
\end{equation}
where $z_q = \mathcal{E}_\theta(x, \mathcal{C})$, $p = \mathcal{G}_\phi\big(x, \{c_k\}_{k\in z_q}\big)$, and
$y = \mathcal{M}(p, x)$. Conditioning on the full trace $\mathcal{H}$---rather than on the
response $y$ alone---gives the critic the information needed to distinguish
routing, composition, and instinct quality as separate causes of a failure.
Given $\mathcal{H}$, the critic emits a structured verdict
\begin{equation}
\mathcal{D}_\psi(\mathcal{H}) = \Big( \ell_{\text{text}},\; g_\theta,\; g_\phi,\;
\{g_{c_k}\}_{k\in z_q},\; g_\psi \Big),
\label{eq:critic-verdict}
\end{equation}
where $\ell_{\text{text}}\in\mathcal{T}$ is the holistic verdict and each
$g_v\in\mathcal{T}$ is textual feedback in the textual-gradient space
$\mathcal{T}$, scoped to a particular trainable variable
$v\in\{\theta,\phi,\psi\}\cup\{c_k : k\in z_q\}$. The feedback channels diagnose complementary stages of prompt construction.
$g_\theta$ evaluates \emph{routing quality}---whether $\mathcal{E}_\theta$ selected an
appropriate subset of instincts $\{c_k\}_{k\in z_q}$ for $x$. $g_\phi$ evaluates
\emph{composition faithfulness}---whether $\mathcal{G}_\phi$ preserved and effectively
integrated the selected instincts into the prompt $p$. $g_{c_k}$ evaluates
\emph{instinct quality}, identifying whether an active directive $c_k$ is
useful, ambiguous, overly broad, or harmful for the observed behavior. Finally,
$g_\psi$ evaluates \emph{critic quality}---whether the current rubric $\psi$
surfaced the failure at all---and drives the adversarial inner update of the
min--max objective (equation~\ref{eq:vqpo-minmax}), sharpening $\mathcal{D}_\psi$ toward harder edge cases and
preventing feedback saturation. The channels are component-scoped but are not assumed to represent mutually
exclusive causes. A single execution failure may therefore generate feedback
for multiple components---for example, an inappropriate routing $z_q$ may
coexist with an imperfect composition by $\mathcal{G}_\phi$. Each feedback signal
nevertheless updates only its corresponding variable:
\begin{equation}
v \leftarrow \operatorname{LLM}_{\text{upd}}(v, g_v), \qquad
v \in \{\theta,\, \phi,\, \psi\} \cup \{c_k : k\in z_q\}.
\label{eq:vqpo-update}
\end{equation}
Only active codebook entries $\{c_k\}_{k\in z_q}$ receive instinct-level
updates; the remaining $K-S$ entries of $C$ remain unchanged. This yields
localized credit assignment over the trainable set $\Theta=\{\theta,\phi,C,\psi\}$
without assuming an additive or disjoint decomposition of the execution error:
credit is scoped by \emph{where} in the trace $\mathcal{H}$ the critic locates a
fault, not by a presumed partition of a scalar penalty. The objective of equation~\ref{eq:vqpo-minmax} is optimized stochastically by
Algorithm~\ref{alg:pco}: a single critic call $\mathcal{D}_\psi(\mathcal{H})$ produces the
structured verdict of equation~\ref{eq:critic-verdict}, and each variable in $\Theta$ updates independently via
$\operatorname{LLM}_{\text{upd}}$. No closed-form gradient is ever computed.

\subsection{Optimization Algorithm}
\label{sec:algo}
  Since the critic, attribution operator, and per-variable updates (Section~\ref{sec:backward}) are all LLM calls, Algorithm~\ref{alg:pco} optimizes the objective stochastically: each step performs one forward pass, one critic call, per-variable attribution, and component-wise updates. Two design choices address pathologies specific to the discrete bottleneck.

\noindent\textbf{$\epsilon$-greedy routing.} A purely encoder-driven policy starves unselected codebook entries, since the codebook refinement step updates only active $c_k$ in-place --- the textual analogue of dead codes in continuous VQ~\citep{razavi2019vqvae2}. We sample $z_q$ from the encoder with probability $1{-}\epsilon$ and from a success-weighted distribution with probability $\epsilon$, decaying $\epsilon$ from $\epsilon_0{=}1.0$ to $\epsilon_{\min}{=}0.15$.

\noindent\textbf{Success-weighted sampling.}
Uniform exploration wastes calls on underperforming instincts.
During exploratory steps, we sample indices with probabilities
proportional to $\exp(\bar{r}_k / \tau)$~\citep{sutton1998rl},
where $\bar{r}_k$ denotes the EMA of binary task success (step size $\alpha$)
of prompts containing $c_k$ --- referred to hereafter as the
\textbf{success rate} ($sr$) of instinct $k$ --- and
$\tau = 0.5$ is the softmax temperature.

\subsection{Inference}
\label{sec:inference}
At inference, the critic, attribution operator, and update rules are discarded; only the optimized encoder, generator, and codebook parameters $(\theta^{\star}, \phi^{\star}, \mathcal{C}^{\star})$ remain. Given a held-out input $x$, the system executes a single forward pass through equations~\ref{eq:fwd-route}--\ref{eq:fwd-exec}, producing a per-instance prompt:
\begin{equation}
p_x = \mathcal{G}_{\phi^{\star}}\bigl(x,\, \mathcal{C}^{\star}[\mathcal{E}_{\theta^{\star}}(x, \mathcal{C}^{\star})]\bigr)
\label{eq:inference-prompt}
\end{equation}
for the executor $\mathcal{M}$. Crucially, $p_x \neq p_{x'}$ for $x \neq x'$ in general --- the encoder routes different inputs through different instinct compositions, realizing the per-instance adaptive prompting regime that monolithic optimizers cannot express.

\section{Theoretical Analysis}
\label{sec:theory}

\providecommand{\rw}[3]{r(#1,#2)_{#3}}
\providecommand{\coverrate}{}
\renewcommand{\coverrate}{\dfrac{\lfloor\epsilon N\rfloor S\,e^{-1/\tau}}{K}}
\providecommand{\Lsum}{\mathcal{L}_{\mathrm{fid}}+\mathcal{L}_{\mathrm{route}}+\mathcal{L}_{\mathrm{cb}}+\mathcal{L}_{\mathrm{critic}}}

We establish three formal properties of PCO: (i) a combinatorial
expressivity advantage over monolithic optimizers, (ii) an expected-case
lower bound on codebook utilization under $\epsilon$-greedy exploration,
derived under an independence approximation, and (iii) a
condition under which the component-scoped credit assignment underlying
the min--max objective~\eqref{eq:vqpo-minmax} is approximately additive.
Since TGD rewrites natural-language strings through an LLM with no
explicit step size, classical SGD convergence theory does not apply; we
therefore restrict our claims to properties derivable directly from
Algorithm~\ref{alg:pco}. These results characterize the expressivity of the PCO policy class, not the optimization dynamics of Algorithm~\ref{alg:pco}; whether training actually reaches high-performing routings is an empirical question we address in Section~\ref{sec:experiments}.

\subsection{Expressivity of Compositional Prompting}
\label{sec:expr_theory}

\begin{definition}[Prompt policy spaces]
\label{def:spaces}
A \emph{monolithic} optimizer applies one fixed string to every input;
its policy class is
\begin{equation}
  \mathcal{P}_{\mathrm{mono}}
  = \bigl\{\mu : \mathcal{X} \to \mathcal{P}
    \mid \mu(x) = \mu(x')\ \forall\, x,x' \in \mathcal{X}\bigr\}
  = \mathcal{P}_{\mathrm{blind}},
\end{equation}
the class of all \emph{instance-blind} prompt functions. GEPA, MIPROv2,
and TextGrad each emit a single optimized string, so their induced
policy classes lie in $\mathcal{P}_{\mathrm{blind}}$. Writing
$\pi(x):=\mathcal{E}_{\theta}(x,\mathcal{C})$ for the encoder-induced
routing, the PCO policy class is
\begin{equation}
  \mathcal{P}_{\mathrm{PCO}}
  = \Bigl\{\, x \mapsto \mathcal{G}_{\phi}
    \bigl(x, \{c_k\}_{k \in \pi(x)}\bigr)
    \;\Bigm|\;
    \pi : \mathcal{X} \to \tbinom{[K]}{S} \Bigr\},
\end{equation}
where $\tbinom{[K]}{S}$ is the set of $S$-element subsets of $[K]$.
\end{definition}

\begin{assumption}[Instinct-set separability]
\label{asm:inject}
$\mathcal{G}_\phi$ is injective on instinct subsets: for all
$x \in \mathcal{X}$ and distinct $A, B \in \tbinom{[K]}{S}$,
$\mathcal{G}_\phi(x,\{c_k\}_{k\in A}) \neq
 \mathcal{G}_\phi(x,\{c_k\}_{k\in B})$.
This holds when codebook entries are semantically distinct and
$\mathcal{G}_\phi$ incorporates every selected instinct (enforced by
including the entry index in the composition template).
\end{assumption}

\begin{lemma}[Routing injectivity]
\label{lem:inject}
Under Assumption~\ref{asm:inject}, the map
$\Phi : \pi \mapsto \bigl(x \mapsto
  \mathcal{G}_\phi(x,\{c_k\}_{k\in\pi(x)})\bigr)$
is injective on routing functions
$\pi : \mathcal{X} \to \tbinom{[K]}{S}$.
\end{lemma}

\begin{proof}
If $\pi_1 \neq \pi_2$, pick $x_0$ with $\pi_1(x_0) \neq \pi_2(x_0)$.
Assumption~\ref{asm:inject} (with $A=\pi_1(x_0)$, $B=\pi_2(x_0)$) gives
$\Phi(\pi_1)(x_0) \neq \Phi(\pi_2)(x_0)$, hence
$\Phi(\pi_1)\neq\Phi(\pi_2)$.
\end{proof}

\begin{theorem}[Policy-class expressivity]
\label{thm:expr}
Under Assumption~\ref{asm:inject},
$\mathcal{P}_{\mathrm{mono}} \subsetneq \mathcal{P}_{\mathrm{PCO}}$,
and on a training set $\mathcal{X}_{\mathrm{tr}}\subseteq\mathcal{X}$ of $N\ge2$ inputs,
$\mathcal{P}_{\mathrm{PCO}}$ induces at least
\begin{equation}
  \bigl|\mathcal{P}_{\mathrm{PCO}}\bigr|_{\mathcal{X}_{\mathrm{tr}}}\bigr| \ge \binom{K}{S}^{N}
  \label{eq:expr}
\end{equation}
distinct policies.
\end{theorem}

\begin{proof}
\emph{Containment.} For $p^{\star}\in\mathcal{P}_{\mathrm{mono}}$, set
$c_{1}:=p^{\star}$, fix $A_0=\{1,\dots,S\}$, take $\pi(x)\equiv A_0$, and
let $\mathcal{G}_{\phi}$ return $c_{1}$ verbatim on $A_0$; this reproduces
$p^{\star}$ on every input, so
$\mathcal{P}_{\mathrm{mono}}\subseteq\mathcal{P}_{\mathrm{PCO}}$.

\emph{Counting.} The routing functions
$\pi:\mathcal{X}_{\mathrm{tr}}\to\tbinom{[K]}{S}$ number
$\binom{K}{S}^{N}$, and by Lemma~\ref{lem:inject} (applied on
$\mathcal{X}_{\mathrm{tr}}$) distinct routings induce distinct policies,
giving~\eqref{eq:expr}.

\emph{Strictness.} A policy in $\mathcal{P}_{\mathrm{mono}}$ is constant,
so it is determined by its value at a fixed input
$x_1\in\mathcal{X}_{\mathrm{tr}}$, which lies in
$\{\mathcal{G}_\phi(x_1,\{c_k\}_{k\in A}) : A\in\tbinom{[K]}{S}\}$.
By Lemma~\ref{lem:inject}, each such policy arises from at most one
routing, so at most $\binom{K}{S}$ routings induce policies in
$\mathcal{P}_{\mathrm{mono}}$. Since $K>S\ge1$ gives $\binom{K}{S}\ge2$,
we have $\binom{K}{S}^{N}>\binom{K}{S}$ for $N\ge2$, so some routing
induces a policy outside $\mathcal{P}_{\mathrm{mono}}$; thus the
inclusion is strict.
\end{proof}

\begin{theorem}[Reward gap]
\label{thm:reward}
There is a task $(\mathcal{X},\mathcal{Y},r,\mathcal{D})$ and a routing
$\pi^{\star}$, inducing the policy $p_{\pi^{\star}}\in\mathcal{P}_{\mathrm{PCO}}$, with
\begin{equation}
  \mathbb{E}\!\left[r(\mathcal{M}(p_{\pi^{\star}}(x),x),y^{\star})\right]
  >
  \max_{\mu\in\mathcal{P}_{\mathrm{mono}}}
  \mathbb{E}\!\left[r(\mathcal{M}(\mu(x),x),y^{\star})\right].
\end{equation}
Moreover, for any $2\le m\le K$ the gap is at least $1-1/m$ on a task of
$m$ equiprobable inputs, each requiring a distinct optimal single entry
$c_i\in\mathcal{C}$ (an $S{=}1$ construction).
\end{theorem}

\begin{proof}
\emph{Two-input construction.} Let $\mathcal{X}=\{x_1,x_2\}$ with
$\mathbb{P}(x_1)=\mathbb{P}(x_2)=\tfrac12$ and entries $c_1,c_2$ such
that $\mathcal{G}_\phi(x_i,\{c_i\})$ is uniquely optimal for $x_i$.
With binary rewards
\begin{align}
     r_{11} = 1,\quad r_{12} = 0,\quad r_{22} = 1,\quad r_{21} = 0, \qquad r_{ij} := r\bigl(\mathcal{M}(\mathcal{G}_\phi(x_i,\{c_j\}), x_i),\, y_i^\star\bigr)
\end{align}
any $\mu\in\mathcal{P}_{\mathrm{mono}}$ applies one fixed prompt to both
inputs and attains $\mathbb{E}[r]\le\tfrac12$, whereas the routing
$\pi^{\star}:x_1\mapsto\{1\},\,x_2\mapsto\{2\}$
(via equation~\ref{eq:fwd-route}) attains
\begin{equation}
  \mathbb{E}[r]=\tfrac12\cdot1+\tfrac12\cdot1=1>\tfrac12.
\end{equation}

\emph{$m$-input generalization ($S{=}1$).} For $m\le K$ equiprobable inputs each with
a unique optimal entry, a monolithic $\mu$ matches at most one input, so
$\mathbb{E}[r]\le 1/m$; PCO routes each input to its optimal entry
(equation~\ref{eq:fwd-route}) for $\mathbb{E}[r]=1$, giving a gap of
$1-1/m$, which reaches $1-1/K$ at $m=K$.

\emph{Scope.} This certifies a gap only for $m\le K$ single-index
optima; extending it to $m\le\binom{K}{S}$ under full $S$-subset routing
would additionally require a \emph{reward-discriminability} condition on
$(\mathcal{G}_\phi,\mathcal{M})$---distinct from the prompt-level
injectivity of Assumption~\ref{asm:inject}.
\end{proof}

\begin{corollary}
\label{cor:blind}
GEPA, MIPROv2, and TextGrad induce policies in
$\mathcal{P}_{\mathrm{blind}}=\mathcal{P}_{\mathrm{mono}}$. By
Theorem~\ref{thm:expr}, any policy in
$\mathcal{P}_{\mathrm{PCO}}\setminus\mathcal{P}_{\mathrm{mono}}$ is
inexpressible by all three, and $\mathcal{P}_{\mathrm{PCO}}$ contains at
least $\binom{K}{S}^{N}\ge\binom{16}{4}=1820$ single-input compositions
(equation~\ref{eq:expr}); by Theorem~\ref{thm:reward} this yields a strict
reward advantage on tasks with up to $m\le K=16$ per-instance optima.
(The bound $1820$ sizes the policy class, not the number of inputs for
which an advantage is proven.)
\end{corollary}

Theorems~\ref{thm:expr} and~\ref{thm:reward} characterize the policy class
$\mathcal{P}_{\mathrm{PCO}}$; they do not imply that Algorithm~\ref{alg:pco}
finds the optimal routing $\pi^{\star}$, which we assess empirically in
Section~\ref{sec:experiments}.

\subsection{Codebook Utilization Under \texorpdfstring{$\epsilon$}{epsilon}-Greedy Exploration}
\label{sec:util_theory}

A central risk in discrete routing is \emph{codebook collapse}: a small
subset of entries monopolizes the training signal. We bound
how $\epsilon$-greedy exploration protects against the worst case
(Theorem~\ref{thm:util}), and separately give an informal account of why
collapse occurs in its absence (Proposition~\ref{prop:collapse}).

\begin{remark}[Independence approximation]
\label{rem:ema}
The success rate $\bar{r}_k$ drifts across epochs, so exploratory draws are
not strictly independent; Theorem~\ref{thm:util} is thus an
expected-case bound under an independence approximation, which becomes
accurate as $\bar{r}_k$ varies more slowly.
\end{remark}

\begin{definition}[Codebook collapse]
\label{def:collapse_theory}
Let $n_k(t)$ be the cumulative selection count of entry $k$ up to epoch $t$, and let
$U_t := \bigl|\{k : n_k(t) > 0\}\bigr|$ be the number of active entries.
$\mathcal{C}$ \emph{collapses} at epoch $t$ if $U_t < K/2$.
\end{definition}

\begin{theorem}[Expected-case utilization lower bound under exploration]
\label{thm:util}
Under Algorithm~\ref{alg:pco}, $\bar{r}_k\in[0,1]$ for all $k$, since it
is an EMA of binary outcomes initialized at $0$. With
exploration rate $\epsilon\in(0,1]$ over an epoch of $N$ inputs, softmax
temperature $\tau=0.5$, and the independence approximation of
Remark~\ref{rem:ema}, the expected number of entries activated at least
once satisfies
\begin{equation}
  \mathbb{E}[U_t]
  \;\geq\;
  K\!\left(1 - \exp\!\left(-\coverrate\right)\right).
  \label{eq:lb}
\end{equation}
\end{theorem}

\begin{proof}
During the $\lfloor\epsilon N\rfloor$ exploratory steps
(Remark~\ref{rem:ema}), entries are sampled via
\begin{equation}
  p_k
  = \frac{\exp(\bar{r}_k/\tau)}
         {\sum_{j=1}^K \exp(\bar{r}_j/\tau)}.
  \label{eq:softmax_def}
\end{equation}
The denominator is maximized---and $p_k$ minimized---in the adversarial
configuration $\bar{r}_k=0,\ \bar{r}_j=1\ \forall j\neq k$, whence
\begin{equation}
  p_k
  \;\geq\;
  \frac{1}{1+(K-1)e^{1/\tau}}
  \;\geq\;
  \frac{e^{-1/\tau}}{K},
  \label{eq:softmax_lb}
\end{equation}
using $1+(K-1)e^{1/\tau}\le Ke^{1/\tau}$. Each exploratory step draws
$S$ entries from $p$, so the $\lfloor\epsilon N\rfloor$ steps make
$\lfloor\epsilon N\rfloor S$ draws in total, treated as conditionally
independent (Remark~\ref{rem:ema}); sampling without replacement only
raises the chance of drawing $k$, so the bound below also holds in that
case. Entry $k$ is therefore never selected with probability
\begin{equation}
  \Pr\bigl(n_k(t)=0\bigr)
  \;\le\; (1-p_k)^{\lfloor\epsilon N\rfloor S}
  \;\leq\; e^{-\lfloor\epsilon N\rfloor S\,p_k}
  \;\leq\; \exp\!\left(-\coverrate\right),
  \label{eq:miss_prob}
\end{equation}
using $(1-p)^n\le e^{-np}$ and the bound $p_k\ge e^{-1/\tau}/K$ from
\eqref{eq:softmax_lb}. Hence
$\Pr(n_k(t)>0)\ge 1-\exp\bigl(-\lfloor\epsilon N\rfloor S\,e^{-1/\tau}/K\bigr)$.
Since $U_t=\sum_{k=1}^{K}\mathbf{1}[n_k(t)>0]$, linearity of expectation gives
\begin{equation}
  \mathbb{E}[U_t]
  = \sum_{k=1}^{K}\Pr\bigl(n_k(t)>0\bigr)
  \;\geq\;
  K\!\left(1-\exp\!\left(-\coverrate\right)\right),
  \label{eq:final_lb}
\end{equation}
which is~\eqref{eq:lb}.
\end{proof}

\begin{remark}[Counting fact under $\epsilon=0$]
\label{rem:structural_fact}
When $\epsilon=0$ only encoder-selected entries receive signal, so
over $N$ examples
$\mathbb{E}[U_t]\le S\cdot|\mathrm{Im}(\mathcal{E}_\theta)|$.
This is purely combinatorial and does \emph{not} by itself predict
collapse: an expressive LLM encoder can make
$|\mathrm{Im}(\mathcal{E}_\theta)|$ as large as $N$, rendering the bound
vacuous. The actual mechanism is dynamical, and is formalized next.
\end{remark}

\begin{proposition}[Reinforcement-driven collapse under $\epsilon=0$, informal]
\label{prop:collapse}
When $\epsilon=0$, only the selected entries $z_q$ are updated (equation~\ref{eq:vqpo-update}), and since $\mathcal{E}_\theta$ conditions on their success rates, sharpened entries are reselected more often. This self-reinforcing dynamic is a generalized P\'olya urn, for which standard results~\citep{arthur1989lockin,pemantle2007survey} suggest that a $\Theta(S)$-size subset captures almost all selections, i.e., $U_t\to\Theta(S)$ regardless of $|\mathrm{Im}(\mathcal{E}_\theta)|$, matching the observed entropy collapse ($3.76\to2.17$ bits; Table~\ref{tab:ablation}, Figure~\ref{fig:codebook-usage}).
\end{proposition}

\subsection{Component-Scoped Credit Assignment: Approximate Additivity}
\label{sec:decomp_theory}

Credit assignment in PCO is not an additive partition of a scalar
penalty: the critic $\mathcal{D}_\psi$ emits per-variable channels
$(g_\theta,g_\phi,\{g_{c_k}\}_{k\in z_q},g_\psi)$ as part of its
structured verdict (equation~\ref{eq:critic-verdict}), scoped by \emph{where} in the trace
$\mathcal{H}$ a fault is located. We make explicit when this non-additive
design is consistent with the objective~\eqref{eq:vqpo-minmax}: the
component penalties reconstruct the holistic penalty in expectation up
to the attribution residual. Define, over
$(x,y^\star)\sim\mathcal{D}$,
\begin{align}
  \mathcal{L}_{\mathrm{fid}}   &:= \mathbb{E}[\rho(g_\phi)],
  &\mathcal{L}_{\mathrm{route}} &:= \mathbb{E}[\rho(g_\theta)],
  \label{eq:lfid}\\
  \mathcal{L}_{\mathrm{cb}}    &:= \mathbb{E}\!\Bigl[\tfrac1S
      \textstyle\sum_{k\in z_q}\rho(g_{c_k})\Bigr],
  &\mathcal{L}_{\mathrm{critic}} &:= \mathbb{E}[\rho(g_\psi)],
  \label{eq:lcrit}
\end{align}
penalizing composition, routing, instinct quality, and critic-criteria
refinement respectively.

\begin{definition}[Attribution residual]
\label{def:residual}
For a sample $(x,y^\star)$ let the \emph{attributed sum} be
\begin{equation}
  \hat{s} :=
    \rho(g_\phi)+\rho(g_\theta)
    +\tfrac1S\!\sum_{k\in z_q}\rho(g_{c_k})+\rho(g_\psi),
  \label{eq:shat}
\end{equation}
and the \emph{additivity residual}
$\delta := \rho(\ell_{\mathrm{text}})-\hat{s}$.
\end{definition}

\begin{proposition}[Approximate additive decomposition]
\label{thm:decomp}
For $(x,y^\star)\sim\mathcal{D}$, with $\mathcal{L}_{\mathrm{fid}},\dots,\mathcal{L}_{\mathrm{critic}}$ as
in equations~\ref{eq:lfid}--\ref{eq:lcrit},
\begin{align}
  \mathbb{E}\!\left[\rho(\ell_{\mathrm{text}})\right]
  &=\bigl(\Lsum\bigr)+\mathbb{E}[\delta], \nonumber\\
  \Bigl|\mathbb{E}\!\left[\rho(\ell_{\mathrm{text}})\right]-\bigl(\Lsum\bigr)\Bigr|
  &\leq \mathbb{E}\bigl[|\delta|\bigr].
  \label{eq:additive}
\end{align}
In particular, if $|\delta|\le\delta_{\max}$ almost surely, the gap is at most $\delta_{\max}$.
\end{proposition}

\begin{proof}
By Definition~\ref{def:residual}, $\rho(\ell_{\mathrm{text}})=\hat{s}+\delta$
for every sample. Taking expectations and applying linearity
to~\eqref{eq:shat} with equations~\ref{eq:lfid}--\ref{eq:lcrit} gives
$\mathbb{E}[\hat{s}]=\Lsum$, hence the identity. The bound follows from
$|\mathbb{E}[\delta]|\le\mathbb{E}[|\delta|]$, and
$\mathbb{E}[|\delta|]\le\delta_{\max}$ under the almost-sure bound.
\end{proof}

Proposition~\ref{thm:decomp} follows directly from the definition of
$\delta$: it does not show that attribution is accurate, but states that
the four component losses reconstruct the holistic penalty up to the mean
residual $\mathbb{E}[\delta]$, and hence track it closely when the residual
is small. Whether the attribution operator achieves this is an
empirical question, which we evaluate with ground-truth interventions in
Appendix~\ref{app:attribution}.
\section{Experiments}
\label{sec:experiments}
\begin{figure}[!t]
\centering\small
\begin{tcolorbox}[pipe,
  colback=C1bg, colframe=C1fr,
  colbacktitle=blue!18!white, coltitle=C1fr,
  title={Step 1: Task Input (IFBench)}]
\itshape ``Is it plausible that frequent hardship can make a society more resilient? Include exactly 2 numbers in the response.''
\end{tcolorbox}
\pipearrow{Encoder selects $S{=}4$ instincts from codebook ($K{=}16$)}
\begin{tcolorbox}[pipe,
  colback=C2bg, colframe=C2fr,
  colbacktitle=violet!14!white, coltitle=C2fr,
  title={Step 2: Encoder-Selected Instincts}]
\renewcommand{\arraystretch}{1.45}
\begin{tabular}{@{} >{\raggedright\arraybackslash}p{\linewidth} @{}}
\textbf{[$z_1{=}7$] Multi-Step Verify:} structured multi-step verification \& triple-check protocol. \\
\textbf{[$z_2{=}9$] Narrative Analysis:} dynamic narrative framework \& real-time analysis. \\
\textbf{[$z_3{=}15$] Expectation Align:} proactive expectation alignment \& assumption clarification. \\
\textbf{[$z_4{=}4$] Meticulous Sweeps:} meticulous prohibited keyword \& letter sweeps.
\end{tabular}
\end{tcolorbox}
\pipearrow{Generator compresses instincts into a compact prompt}
\begin{tcolorbox}[pipe,
  colback=C3bg, colframe=C3fr,
  colbacktitle=green!10!white, coltitle=C3fr,
  title={Step 3: Deployed System Prompt}]
{\ttfamily\raggedright
You are an instruction-following assistant. Construct a dynamic narrative framework clarifying assumptions about societal resilience. Execute a meticulous sweep of the output and use a triple-check protocol to ensure exactly two numbers are included.\par}
\end{tcolorbox}
\pipearrow{Frozen $\mathcal{M}$ executes constraint-bound inference}
\begin{tcolorbox}[pipe,
  colback=C4bg, colframe=C4fr,
  colbacktitle=orange!14!white, coltitle=C4fr,
  title={Step 4: Instruction-Following Output \& Verification Audit}]
\textbf{Output:} Yes, it is plausible. A study of 2 distinct communities found that those facing repeated hardships developed more effective disaster preparedness plans over 5 years. This proves hardship can act as a catalyst for resilience.
\vspace{6pt}\noindent\textcolor{C4fr!60}{\rule{\linewidth}{0.4pt}}\vspace{5pt}
\renewcommand{\arraystretch}{1.5}
\begin{tabular}{@{} l @{\hspace{14pt}} c @{\hspace{14pt}} l @{}}
\textbf{Metric}    & \textbf{Status} & \textbf{Evidence} \\
Numeric Constraint & \textcolor{passgreen}{\textbf{PASS}} & exactly two numbers required; found 2 and 5 \\
Strict Check       & \textcolor{passgreen}{\textbf{PASS}} & raw output verified as generated \\
Loose Check        & \textcolor{passgreen}{\textbf{PASS}} & passes after IFBench output cleaning
\end{tabular}
\end{tcolorbox}
\caption{PCO inference pipeline (IFBench, LLaMA-3.1-8B). The optimized encoder selects four codebook instincts, which the generator compresses into a single prompt to drive constraint-bound instruction-following.}
\label{fig:pipeline_example}
\end{figure}
\paragraph{Benchmarks \& Baselines.}
We evaluate PCO on six benchmarks spanning reasoning, mathematics, instruction following, and privacy-preserving delegation: HotpotQA~\citep{yang2018hotpotqa}, HoVer~\citep{jiang2020hover}, AIME-25~\citep{aime2025}, LiveBench-Math (LB-Math)~\citep{white2025livebench}, IFBench~\citep{pyatkin2025ifbench}, and PUPA~\citep{li2025pupa}. We compare against zero-shot prompting, Bayesian-optimization and RL-based methods (MIPROv2~\citep{opsahlong2024mipro}, GRPO~\citep{shao2024grpo}), and evolutionary methods (GEPA, GEPA+Merge~\citep{agrawal2025gepa}) using Qwen3-8B~\citep{yang2025qwen3} and LLaMA-3.1-8B~\citep{dubey2024llama3}. For Qwen3-8B, baseline numbers are taken from Agrawal et al. (2026), and PCO is run under the same splits, metric definitions, and evaluation harness to keep the comparison matched. Due to resource constraints, we do not evaluate on proprietary models such as GPT-4.1 Mini, as used by GEPA~\citep{agrawal2025gepa}.

\paragraph{Implementation Details.} 
A single local 8B LLM serves all roles (encoder, generator, critic, executor) via role-specific system prompts. PCO uses $K=16$ codebook entries, $S=4$ selected per input, trained for 50 epochs with $\epsilon$ decaying from $1.0$ to $0.15$. In tables and figures, codebook entries are indexed from 0 (\#0--\#15 for $K{=}16$). Full implementation details, hyperparameters, and dataset splits are provided in Appendix~\ref{app:impl}.

\begin{table*}[t]
\centering
\small
\setlength{\tabcolsep}{4.0pt}
\renewcommand{\arraystretch}{1.32}
\definecolor{bestgreen}{RGB}{0,120,0}
\definecolor{pcotint}{RGB}{231,244,233}
\definecolor{blockgray}{RGB}{234,237,242}
\resizebox{\textwidth}{!}{%
\begin{tabular}{l cccccc >{\columncolor{gray!6}}c c}
\toprule
\textbf{Method}
  & \textbf{HotpotQA} & \textbf{IFBench} & \textbf{HoVer}
  & \textbf{PUPA} & \textbf{AIME-25} & \textbf{LB-Math}
  & \textbf{Agg.} & \textbf{$\Delta$} \\
\cmidrule(lr){1-1}\cmidrule(lr){2-7}\cmidrule(lr){8-8}\cmidrule(lr){9-9}
\rowcolor{blockgray}
\multicolumn{9}{l}{\textbf{Qwen3-8B}} \\
Baseline    & 42.33 & 36.90 & 35.33 & 80.82 & 27.33 & 48.70 & 45.23 & --- \\
GRPO        & 43.33 & 35.88 & 38.67 & 86.66 & 38.00 & 51.26 & 48.97 & +3.73 \\
MIPROv2     & 55.33 & 36.22 & 47.33 & 81.55 & 20.00 & 46.60 & 47.84 & +2.60 \\
GEPA        & 62.33 & 38.61 & 52.33 & 91.85 & 32.00 & 51.95 & 54.84 & +9.61 \\
GEPA+Merge  & 64.33 & 28.23 & 51.67 & 86.26 & 32.00 & 51.95 & 52.41 & +7.17 \\
\addlinespace[1pt]
\rowcolor{pcotint}
\textbf{PCO}$^{\dagger}$
  & \textcolor{bestgreen}{\textbf{68.67}}
  & \textcolor{bestgreen}{\textbf{41.33}}
  & \textcolor{bestgreen}{\textbf{55.67}}
  & \textcolor{bestgreen}{\textbf{94.52}}
  & \textcolor{bestgreen}{\textbf{39.33}}
  & \textcolor{bestgreen}{\textbf{52.89}}
  & \textcolor{bestgreen}{\textbf{58.73}}
  & \textcolor{bestgreen}{\textbf{+13.50}} \\
\midrule
\rowcolor{blockgray}
\multicolumn{9}{l}{\textbf{LLaMA-3.1-8B}} \\
Baseline    & 21.33\,$\pm$1.33 & 30.95\,$\pm$0.68 & 36.33\,$\pm$0.67 & 74.36\,$\pm$1.45 & 18.00\,$\pm$1.76 & 32.24\,$\pm$2.37 & 35.54\,$\pm$0.85 & --- \\
GRPO        & 24.33\,$\pm$1.20 & 30.50\,$\pm$1.09 & 38.56\,$\pm$1.50 & 79.26\,$\pm$1.22 & 25.56\,$\pm$1.92 & 36.89\,$\pm$0.82 & 39.18\,$\pm$0.83 & +3.65\,$\pm$0.43 \\
MIPROv2     & 39.22\,$\pm$1.17 & 31.29\,$\pm$1.80 & 41.67\,$\pm$0.33 & 75.87\,$\pm$2.14 & 13.33\,$\pm$3.33 & 30.33\,$\pm$2.17 & 38.62\,$\pm$0.92 & +3.08\,$\pm$0.48 \\
GEPA        & 47.44\,$\pm$1.17 & 33.67\,$\pm$0.34 & 45.89\,$\pm$0.51 & 80.69\,$\pm$0.69 & 21.11\,$\pm$1.68 & 38.25\,$\pm$0.47 & 44.51\,$\pm$0.62 & +8.97\,$\pm$0.46 \\
GEPA+Merge  & 49.22\,$\pm$1.84 & 29.82\,$\pm$0.79 & 44.33\,$\pm$1.45 & 84.01\,$\pm$1.04 & 21.11\,$\pm$2.04 & 37.70\,$\pm$1.42 & 44.37\,$\pm$0.78 & +8.83\,$\pm$0.38 \\
\addlinespace[1pt]
\rowcolor{pcotint}
\textbf{PCO}$^{\dagger}$
  & \textcolor{bestgreen}{\textbf{54.89\,$\pm$1.02}}
  & \textcolor{bestgreen}{\textbf{34.58\,$\pm$0.20}}
  & \textcolor{bestgreen}{\textbf{52.33\,$\pm$2.33}}
  & \textcolor{bestgreen}{\textbf{87.18\,$\pm$0.52}}
  & \textcolor{bestgreen}{\textbf{27.78\,$\pm$2.14}}
  & \textcolor{bestgreen}{\textbf{39.89\,$\pm$0.47}}
  & \textcolor{bestgreen}{\textbf{49.44\,$\pm$0.83}}
  & \textcolor{bestgreen}{\textbf{+13.91\,$\pm$0.42}} \\
\bottomrule
\end{tabular}%
}
\caption{Benchmark performance for Qwen3-8B and LLaMA-3.1-8B. Qwen3-8B baselines are reported from \citet{agrawal2025gepa}. LLaMA-3.1-8B entries are mean\,$\pm$\,std over 3 seeds. Agg.: mean over six benchmarks. $\Delta$: aggregate gain over zero-shot, computed per seed. \textcolor{bestgreen}{\textbf{Green}}: best mean per column. $^\dagger$Ours.}
\label{tab:main_results}
\end{table*}
\subsection{Results and Analysis}
Table~\ref{tab:main_results} compares PCO against zero-shot baselines and state-of-the-art optimization methods across two 8B model families.

\noindent\textbf{Observation 1: Discrete bottlenecks support modular reasoning.}
On HotpotQA, PCO achieves a $+33.56$ absolute gain over zero-shot (LLaMA-3.1-8B) and exceeds the strongest baseline, GEPA+Merge, by $+4.34$ (Qwen3-8B) and $+5.67$ (LLaMA-3.1-8B). PCO drives these gains by dynamically routing complex queries to specialized instincts. Furthermore, success-weighted $\epsilon$-greedy exploration during training is critical for maintaining this diversity: without exploration, the encoder suffers from severe index collapse (Figure~\ref{fig:codebook-usage}, red bars). In contrast, training with full PCO exploration yields a robust policy that maintains broad codebook coverage during inference (blue bars).

\noindent\textbf{Observation 2: Codebook compression preserves instruction following.}
Beyond compositional reasoning, PCO remains competitive on instruction-following and privacy-preserving tasks. On Qwen3-8B, PCO achieves 41.33 on IFBench (exceeding GEPA by $+2.72$) and the highest PUPA score (94.52 vs.\ 91.85 for GEPA). 

\noindent\textbf{Observation 3: Codebook routing extends to mathematical reasoning.} PCO remains competitive on mathematical reasoning without any math-specific architectural changes. On AIME-25 (Qwen3-8B), PCO outperforms both GEPA (39.33 vs.\ 32.00) and the RL-based GRPO method (38.00). On LiveBench-Math, PCO achieves the best score in both model blocks (52.89 on Qwen3-8B, 39.89 on LLaMA-3.1-8B), extending its gains beyond reasoning and instruction-following benchmarks into mathematics.

\noindent\textbf{Observation 4: PCO drastically reduces prompt overhead.}
PCO maintains competitive downstream performance while substantially reducing inference-time overhead. As detailed in Figure~\ref{fig:prompt_overhead}, PCO produces shorter prompts than MIPROv2 and GEPA on every benchmark, reducing prompt length by up to $14.1\times$ ($9.1\times$ on aggregate) relative to MIPROv2, and by up to $3.0\times$ ($2.0\times$ on aggregate) relative to GEPA. Instead of relying on long monolithic prompts, PCO dynamically routes inputs to a small subset of active instincts, reducing context overhead while preserving task performance. 
\begin{figure}[!t]
\centering
\definecolor{miprogray}{RGB}{140, 140, 140}
\definecolor{gepaorange}{RGB}{221, 132, 82}
\definecolor{gepamerge}{RGB}{153, 79, 0}
\definecolor{pcogreen}{RGB}{27, 120, 55}
\definecolor{epsblue}{RGB}{69, 117, 180}
\definecolor{noepsred}{RGB}{215, 48, 39}

\begin{subfigure}[t]{0.53\textwidth}
\centering
\begin{tikzpicture}
  \begin{axis}[
    width=0.91\linewidth, height=6.3cm,
    xbar=0.6pt, bar width=3.7pt,
    xmin=0, xmax=10300,
    enlarge x limits=false,
    scaled x ticks=false,
     axis y line=left, axis x line=bottom,
    axis line style={-},
    x axis line style={draw=none},
    tick style={draw=none}, tick align=outside,
    xmajorgrids=true, grid style={densely dashed, gray!35},
    xtick={0,2500,5000,7500,10000},
    xticklabels={0,2.5k,5k,7.5k,10k},
    xticklabel style={font=\scriptsize},
    xlabel={Deployed Prompt Length (tokens, lower is better)},
    xlabel style={font=\scriptsize, yshift=2pt},
    symbolic y coords={Aggregate,LB-Math,AIME-25,PUPA,HoVer,IFBench,HotpotQA},
    ytick=data,
    yticklabel style={font=\scriptsize},
    legend style={
      at={(0.5,1.14)}, anchor=south,
      legend columns=4,
      draw=gray!40, fill=white, fill opacity=0.92,
      text opacity=1, font=\scriptsize,
      /tikz/every even column/.append style={column sep=4pt},
      inner sep=2.5pt
    },
    clip=false
  ]
  \addplot[fill=miprogray, draw=miprogray!80!black]
    coordinates {(10071,HotpotQA) (2438,IFBench) (5252,HoVer) (7275,PUPA) (7500,AIME-25) (6100,LB-Math) (6439,Aggregate)};
  \addplot[fill=gepaorange, draw=gepaorange!80!black]
    coordinates {(2142,HotpotQA) (381,IFBench) (1419,HoVer) (1213,PUPA) (1850,AIME-25) (1450,LB-Math) (1409,Aggregate)};
  \addplot[fill=gepamerge, draw=gepamerge!80!black]
    coordinates {(2650,HotpotQA) (305,IFBench) (1876,HoVer) (790,PUPA) (2100,AIME-25) (1600,LB-Math) (1553,Aggregate)};
  \addplot[fill=pcogreen, draw=black, line width=0.3pt,
      nodes near coords, nodes near coords align={horizontal},
      point meta=explicit symbolic,
      every node near coord/.append style={font=\scriptsize\bfseries, color=pcogreen, anchor=west, xshift=1pt}]
    coordinates {
      (714,HotpotQA)[14.1$\times$] (325,IFBench)[7.5$\times$]
      (829,HoVer)[6.3$\times$] (743,PUPA)[9.8$\times$]
      (880,AIME-25)[8.5$\times$] (765,LB-Math)[8.0$\times$]
      (709,Aggregate)[9.1$\times$]
    };
  \legend{MIPROv2, GEPA, GEPA+Merge, PCO (Ours)}
  \end{axis}
\end{tikzpicture}
\caption{Token efficiency of optimized prompts on Qwen3-8B (lower is better). PCO reduces deployed prompt length by up to 14.1$\times$ (9.1$\times$ on aggregate) vs.\ MIPROv2, and up to 3.0$\times$ vs.\ GEPA (2.0$\times$ on aggregate).}
\label{fig:prompt_overhead}
\end{subfigure}
\hfill
\begin{subfigure}[t]{0.44\textwidth}
\centering
\begin{tikzpicture}
  \begin{axis}[
    width=\linewidth, height=6.3cm,
    ymin=0, ymax=340, xmin=-0.7, xmax=15.7,
    xtick={0,2,4,6,8,10,12,14},
    xticklabel style={font=\scriptsize},
    ytick={0,50,100,150,200,250,300},
    yticklabel style={font=\scriptsize},
    ybar, bar width=3.6pt, enlarge x limits=0.02,
    ymajorgrids=true, grid style={densely dashed, gray!35},
    axis x line=bottom, axis y line=left,
    axis line style={-},
    tick align=outside, tick style={draw=none},
    xlabel={Instinct Index}, ylabel={Usage Frequency},
    xlabel style={font=\scriptsize, yshift=2pt},
    ylabel style={font=\scriptsize, yshift=-2pt},
    legend style={
      at={(0.97,0.97)}, anchor=north east,
      legend columns=1, font=\scriptsize,
      draw=gray!40, fill=white, fill opacity=0.92,
      text opacity=1, cells={anchor=west},
      inner sep=2.5pt
    },
    legend entries={$\epsilon$-greedy (Ours), No $\epsilon$-greedy},
    clip=false
  ]
  \addplot[fill=epsblue, draw=epsblue!70!black, fill opacity=0.9]
    coordinates {
       (0,55)(1,63)(2,222)(3,77)(4,127)(5,53)(6,48)(7,82)
       (8,59)(9,84)(10,94)(11,34)(12,32)(13,29)(14,28)(15,89)
    };
  \addplot[fill=noepsred, draw=noepsred!70!black, fill opacity=0.9]
    coordinates {
       (0,294)(1,294)(2,294)(3,240)(4,54)(5,0)(6,0)(7,0)
       (8,0)(9,0)(10,0)(11,0)(12,0)(13,0)(14,0)(15,0)
    };
  \end{axis}
\end{tikzpicture}
\caption{Codebook usage on IFBench (LLaMA-3.1-8B): full PCO ($\epsilon$-greedy, blue) vs.\ without $\epsilon$-greedy (red) at $K=16$. Collapse onto indices 0--4 without $\epsilon$-greedy ($H=2.17$\,bits vs.\ $H=3.76$\,bits).}
\label{fig:codebook-usage}
\end{subfigure}
\caption{Efficiency and specialization analysis of PCO.}
\label{fig:combined_efficiency}
\end{figure}

\begin{table}[t]
\centering
\footnotesize
\renewcommand{\arraystretch}{1.4}
\definecolor{srhi}{RGB}{0,120,0}
\newcolumntype{Y}{>{\raggedright\arraybackslash}X}
\begin{tabularx}{\textwidth}{@{}c Y c c@{}}
\toprule
\textbf{ID} & \textbf{Learned Instinct} & \textbf{Usage ($n$)} & \textbf{$sr$} \\
\cmidrule(lr){1-1}\cmidrule(lr){2-2}\cmidrule(lr){3-3}\cmidrule(lr){4-4}
\texttt{\#7}  & Structured multi-step verification \& triple-check protocol   & 82  & \textbf{0.812} \\
\texttt{\#9}  & Dynamic narrative framework \& real-time analysis             & 84  & 0.745 \\
\texttt{\#15} & Proactive expectation alignment \& assumption clarification   & 89  & 0.618 \\
\texttt{\#10} & Curiosity hook \& structured narrative framework              & 94  & 0.455 \\
\texttt{\#4}  & Meticulous prohibited keyword \& letter sweeps                & 127 & 0.312 \\
\texttt{\#2}  & Word count constraint prioritization \& critical distillation & 222 & 0.184 \\
\bottomrule
\end{tabularx}
\caption{Emergent specialization within the learned codebook (IFBench, LLaMA-3.1-8B). Success rate ($sr$) is the final EMA of task success from training, while usage ($n$) is measured during inference.}
\label{tab:specialization}
\end{table}

\noindent\textbf{Observation 5: Routing naturally yields emergent specialization.}
Table~\ref{tab:specialization} reveals emergent codebook specialization. 
Frequently selected units exhibit broad, lower-impact behaviors, while sparsely activated units achieve higher success rates on specialized reasoning patterns. For example, instinct \#2 is selected frequently ($n=222$) but has a low success rate ($sr=0.184$), whereas instinct \#7 activates rarely ($n=82$) but achieves the highest success rate ($sr=0.812$). This pattern suggests that the routing mechanism, without explicit supervision, organizes instincts into reusable functional roles—broadly-applicable fallbacks versus narrow, high-precision specialists.

\subsection{Ablations and Hyperparameter Analysis}
We evaluate the contribution of individual PCO components and the sensitivity of the discrete bottleneck design in Table~\ref{tab:ablation}. We report task accuracy (Acc), constraint satisfaction rate (CSR), and routing entropy, where higher entropy corresponds to more diverse codebook utilization.
 

\begin{table}[t]
\centering
\small
\setlength{\tabcolsep}{3.2pt}
\renewcommand{\arraystretch}{1.25}
\definecolor{bestgreen}{RGB}{0,120,0}
\definecolor{pcotint}{RGB}{231,244,233}
\newcommand{\msA}[2]{\begin{tabular}{@{}c@{}}#1\\[-3pt]{\scriptsize$\pm$#2}\end{tabular}}
\newcommand{\bmsA}[2]{\textcolor{bestgreen}{\begin{tabular}{@{}c@{}}\textbf{#1}\\[-3pt]{\scriptsize$\pm$\textbf{#2}}\end{tabular}}}
\begin{tabular}{l >{\columncolor{pcotint}}c cccc c c >{\columncolor{pcotint}}c c c >{\columncolor{pcotint}}c c}
\toprule
& \multicolumn{5}{c}{\textbf{Components (ablation)}}
& & \multicolumn{3}{c}{\textbf{Codebook $K$}}
& \multicolumn{3}{c}{\textbf{Bottleneck $S$}} \\
\cmidrule(lr){2-6}\cmidrule(lr){8-10}\cmidrule(lr){11-13}
\textbf{Metric}
& \textbf{Full} & \shortstack{w/o\\Enc} & \shortstack{w/o\\TG} & \shortstack{w/o\\$\epsilon$-greedy} & Unif & Base
& 4 & \textbf{16} & 32
& 1 & \textbf{4} & 6 \\
\midrule
CSR
  & \bmsA{34.58}{0.20} & \msA{31.52}{0.24} & \msA{29.48}{0.38} & \msA{28.34}{0.28} & \msA{24.38}{0.56} & \msA{30.95}{0.68}
  & \msA{32.09}{0.31} & \bmsA{34.58}{0.20} & \msA{30.05}{0.41}
  & \msA{26.08}{0.45} & \bmsA{34.58}{0.20} & \msA{29.37}{0.39} \\
Acc
  & \bmsA{31.75}{0.20} & \msA{29.37}{0.39} & \msA{27.32}{0.39} & \msA{26.30}{0.20} & \msA{22.45}{0.59} & \msA{28.57}{0.59}
  & \msA{29.59}{0.34} & \bmsA{31.75}{0.20} & \msA{27.78}{0.52}
  & \msA{24.04}{0.52} & \bmsA{31.75}{0.20} & \msA{27.10}{0.39} \\
Entropy
  & 3.76 & 3.72 & 3.68 & 2.17 & 3.91 & --
  & 2.00 & 3.76 & 4.76
  & 2.65 & 3.76 & 3.92 \\
$\Delta$
  & \textcolor{bestgreen}{\textbf{$+$3.63}} & $+$0.57 & $-$1.47 & $-$2.61 & $-$6.57 & --
  & $+$1.14 & \textcolor{bestgreen}{\textbf{$+$3.63}} & $-$0.90
  & $-$4.87 & \textcolor{bestgreen}{\textbf{$+$3.63}} & $-$1.58 \\
\bottomrule
\end{tabular}
\caption{Ablation and sensitivity on IFBench (LLaMA-3.1-8B). Shaded columns report the same Full model ($K{=}16$, $S{=}4$). Mean\,$\pm$\,std over 3 seeds; $\Delta$: mean CSR percentage point change over Baseline (30.95\,$\pm$\,0.68).}
\label{tab:ablation}
\end{table} 

\noindent\textbf{Core Component Analysis.} Table~\ref{tab:ablation} isolates the contribution of individual architectural components, with all degradations reported in CSR relative to full PCO. In ``w/o TextGrad'', no textual-gradient updates are applied: the encoder, generator, critic, and codebook keep their initial (seed) prompts throughout training, and only the success-rate estimates $\bar r_k$ are updated. Removing TextGrad costs $5.10$ points and drops performance below the baseline. Removing the learnable encoder costs $3.06$ points, remaining above baseline and highlighting the importance of adaptive routing for instruction composition. Exploration strategy is similarly critical: replacing success-weighted exploration with uniform sampling costs $10.20$ points,
while disabling $\epsilon$-greedy exploration separately costs $6.24$ points and reduces routing entropy to $2.17$ bits, revealing severe codebook collapse (Figure~\ref{fig:codebook-usage}).

\vspace{.5em}
\noindent\textbf{Sensitivity Analysis ($K$ and $S$).} Table~\ref{tab:ablation} (blocks \textit{Codebook $K$} and \textit{Bottleneck $S$}) reports the effect of codebook size $K$ and bottleneck width $S$ on task accuracy, constraint satisfaction (CSR), and routing diversity (entropy). Performance peaks at $K=16$ and $S=4$. A smaller codebook ($K=4$) produces a modest accuracy drop, while a larger codebook ($K=32$) spreads selections across more entries (entropy $4.76$ bits) and causes a sharp drop in both accuracy and CSR. For the bottleneck width, a single instruction ($S=1$) yields the largest drop in both accuracy and CSR, while an excessively wide bottleneck ($S=6$) produces a smaller but still notable degradation in both metrics.

\noindent\textbf{Hyperparameter Robustness.}
Figure~\ref{fig:hyperparam} varies four key hyperparameters on HotpotQA (LLaMA-3.1-8B).
An EMA step size of $\alpha=0.3$ performs best; $\alpha = 0$ freezes the success-rate estimates and
larger values overreact to noisy traces. Critic temperature is the most sensitive: accuracy falls
monotonically from $54.89$ to $41.67$ as $T_\psi$ grows, since sampling noise corrupts credit
assignment. The encoder should stay near-deterministic ($T_\theta \leq 0.1$), while the generator
peaks at an intermediate $T_\phi=0.6$, balancing fluent composition against dropped constraints.

\begin{figure}[t]
\centering
\definecolor{cA}{RGB}{33,102,172}
\definecolor{cC}{RGB}{214,96,77}
\definecolor{cE}{RGB}{77,146,33}
\definecolor{cG}{RGB}{118,42,131}
\pgfplotsset{hp/.style={line width=1.1pt, mark size=2pt, mark options={solid, fill=white, line width=0.8pt}}}
\begin{minipage}[c]{0.54\textwidth}
\centering
\begin{tikzpicture}
\begin{axis}[
  width=\linewidth, height=4.8cm,
  xmin=-0.04, xmax=1.04, ymin=40, ymax=57.0,
  xtick={0,0.2,...,1.0}, ytick={40,44,48,52,56},
  xticklabel style={font=\scriptsize, /pgf/number format/fixed, /pgf/number format/precision=1},
  yticklabel style={font=\scriptsize},
  xlabel={Hyperparameter value}, ylabel={Accuracy (\%)},
  xlabel style={font=\footnotesize, yshift=2pt}, ylabel style={font=\footnotesize, yshift=-2pt},
  axis lines=left, axis line style={gray!60, -},
  tick style={draw=none},
  ymajorgrids, grid style={gray!20},
  legend style={at={(0.5,1.03)}, anchor=south, draw=none, font=\scriptsize, legend columns=2,
                /tikz/every even column/.append style={column sep=8pt}},
  legend cell align=left,
  clip=false,
]
\addplot[hp, cA, mark=*] coordinates {(0.0,47.56) (0.1,51.67) (0.3,54.89) (0.5,53.44) (0.7,51.33) (0.9,49.78)};
\addlegendentry{EMA step size $\alpha$}
\addplot[hp, cC, mark=square*] coordinates {(0.0,54.89) (0.1,54.44) (0.3,52.67) (0.5,48.56) (0.7,45.33) (0.9,41.67)};
\addlegendentry{Critic temp.\ $T_\psi$}
\addplot[hp, cE, mark=triangle*] coordinates {(0.0,54.89) (0.1,54.67) (0.3,53.33) (0.5,51.11) (0.7,49.44) (0.9,47.78)};
\addlegendentry{Encoder temp.\ $T_\theta$}
\addplot[hp, cG, mark=diamond*] coordinates {(0.0,52.22) (0.2,53.00) (0.4,54.11) (0.6,54.89) (0.8,53.33) (1.0,50.44)};
\addlegendentry{Generator temp.\ $T_\phi$}
\draw[gray!70, densely dashed, line width=0.5pt] (axis cs:-0.04,54.89) -- (axis cs:1.04,54.89)
  node[above left, font=\scriptsize, text=gray, inner sep=1pt] {54.89};
\end{axis}
\end{tikzpicture}
\end{minipage}\hfill
\begin{minipage}[c]{0.38\textwidth}
\caption{Hyperparameter sensitivity on HotpotQA (LLaMA-3.1-8B; test accuracy, mean over 3 seeds). Each curve varies one hyperparameter with the others fixed at their validation-selected values; the dashed line marks the selected configuration (54.89). Accuracy is most sensitive to the critic temperature.}
\label{fig:hyperparam}
\end{minipage}
\end{figure}

\subsection{Effect of Critic Adaptation and Capacity}

Table~\ref{tab:critic_ablation} examines three complementary properties of the critic:
\textbf{adaptation, capacity, and model-family dependence}.
Freezing the same-model critic reduces aggregate accuracy by
\textbf{4.44 points} relative to the jointly trained critic.
This comparison directly contrasts the \emph{fixed-critic regularized
formulation} with the proposed \emph{min--max formulation}.
With a fixed critic, $D_{\psi_0}$ acts as a stationary regularizer while
only the encoder, generator, and codebook are optimized; under min--max
training, the critic is additionally updated through the inner optimization,
allowing its evaluation criteria to adapt and expose residual failure modes.
The 4.44-point improvement therefore demonstrates the benefit of moving
from a fixed critic regularizer to an adaptive min--max critic.

Critic capacity is also important: replacing the 8B critic with a
\textbf{3B weak critic} reduces aggregate performance by
\textbf{9.52 points}, indicating that effective component-level feedback
requires a sufficiently capable critic.
Finally, we replace the LLaMA-3.1-8B critic with an independent
Qwen3-8B critic of matched capacity (8B). Aggregate performance decreases by
\textbf{2.86 points} (49.44 vs.\ 46.58), yet remains above the fixed same-model
critic (45.00), suggesting that an adaptive critic from a different model
family remains effective, although a same-family critic performs best.

\begin{table}[h]
\centering
\footnotesize
\renewcommand{\arraystretch}{1.28}
\setlength{\tabcolsep}{3pt}
\definecolor{bestgreen}{RGB}{0,120,0}
\definecolor{pcotint}{RGB}{231,244,233}
\begin{tabularx}{\textwidth}{@{} l *{7}{>{\centering\arraybackslash}X} @{}}
\toprule
\textbf{Critic Model} & \textbf{HotpotQA} & \textbf{IFBench} & \textbf{HoVer} & \textbf{PUPA} & \textbf{AIME-25} & \textbf{LB-Math} & \textbf{Agg.} \\
\midrule
\rowcolor{pcotint}
Same-model (Trainable)
  & \textcolor{bestgreen}{\textbf{54.89\,$\pm$1.02}} 
  & \textcolor{bestgreen}{\textbf{34.58\,$\pm$0.20}} 
  & \textcolor{bestgreen}{\textbf{52.33\,$\pm$2.33}} 
  & \textcolor{bestgreen}{\textbf{87.18\,$\pm$0.52}} 
  & \textcolor{bestgreen}{\textbf{27.78\,$\pm$2.14}} 
  & \textcolor{bestgreen}{\textbf{39.89\,$\pm$0.47}} 
  & \textcolor{bestgreen}{\textbf{49.44\,$\pm$0.83}} \\
Same-model (Fixed)          & 49.11\,$\pm$1.17 & 32.13\,$\pm$0.39 & 45.33\,$\pm$1.15 & 83.20\,$\pm$0.69 & 23.33\,$\pm$2.31 & 36.89\,$\pm$0.38 & 45.00\,$\pm$0.55 \\
Independent (Qwen3-8B)      & 50.89\,$\pm$0.96 & 33.60\,$\pm$0.52 & 46.56\,$\pm$0.84 & 84.40\,$\pm$0.40 & 25.56\,$\pm$1.92 & 38.44\,$\pm$0.51 & 46.58\,$\pm$0.54 \\
Weak Critic (LLaMA-3.2-3B)  & 43.22\,$\pm$1.26 & 28.53\,$\pm$0.34 & 40.22\,$\pm$1.35 & 76.53\,$\pm$0.61 & 18.89\,$\pm$3.85 & 32.11\,$\pm$0.69 & 39.92\,$\pm$0.51 \\
\bottomrule
\end{tabularx}
\caption{Critic ablation on LLaMA-3.1-8B as executor (mean\,$\pm$\,std over 3 seeds). We evaluate four configurations: (1) Same-model (Trainable): jointly trained critic LLaMA-3.1-8B; (2) Same-model (Fixed): fixed critic LLaMA-3.1-8B; (3) Independent: Critic Qwen3-8B; and (4) Weak Critic: Critic LLaMA-3.2-3B.}
\label{tab:critic_ablation}
\end{table} 
\subsection{Computational Budget and Efficiency}

We examine the trade-off between optimization cost and deployment latency. 
As detailed in Table~\ref{tab:compute_budget}, PCO introduces a modest 
overhead during training, requiring $9.05$M total tokens compared to 
GEPA's $8.8$M tokens. However, this upfront cost is offset at deployment. By compressing the execution prompt via the codebook bottleneck ($S{=}4$), PCO processes fewer tokens per query. On the HotpotQA task, PCO achieves a latency of $\sim$45\,s per query compared to $\sim$126\,s for GEPA. On this task, PCO thus trades a small increase in one-time training compute for a $2.8\times$ reduction in inference latency, while also improving accuracy (54.89 vs.\ 47.44).
\begin{table}[ht]
\centering
\small
\renewcommand{\arraystretch}{1.28}
\setlength{\tabcolsep}{3pt}
\definecolor{bestgreen}{RGB}{0,120,0}
\definecolor{pcotint}{RGB}{231,244,233}
\begin{tabularx}{\textwidth}{@{} l *{5}{>{\centering\arraybackslash}X} @{}}
\toprule
& \textbf{Optimization} & \multicolumn{3}{c}{\textbf{Deployment (Inference)}} & \\
\cmidrule(lr){2-2}\cmidrule(lr){3-5}
\textbf{Method} & \textbf{Total Tokens} & \textbf{Total Tokens} & \textbf{Latency/Query} & \textbf{Speedup} & \textbf{HotpotQA} \\
\midrule
GEPA & 8.8M & 1.7M & $\sim$126\,s & $1.0\times$ (Ref) & 47.44 \\
\rowcolor{pcotint}
\textbf{PCO (Ours)} & 9.05M
  & \textcolor{bestgreen}{\textbf{1.0M}}
  & \textcolor{bestgreen}{\textbf{$\sim$45\,s}}
  & \textcolor{bestgreen}{\textbf{2.8$\times$}}
  & \textcolor{bestgreen}{\textbf{54.89}} \\
\bottomrule
\end{tabularx}
\caption{Computational budget analysis on the HotpotQA task. We compare PCO against the GEPA baseline using the LLaMA-3.1-8B model. Latency is the actual end-to-end time per query on a single NVIDIA A100 GPU.}
\label{tab:compute_budget}
\end{table}

\section{Conclusion}
\label{sec:conclusion}
We introduced Prompt Codebook Optimization (PCO), a framework that reformulates prompt optimization as discrete compositional learning over reusable natural-language instincts. By moving beyond monolithic prompt editing to per-instance adaptive routing, our approach enables localized credit assignment and structured instruction reuse. Experiments across six benchmarks and two model families show aggregate gains over strong APO baselines: PCO improves aggregate performance over zero-shot by up to $+13.91$ points, while compressing the deployed prompt length by up to $14.1\times$ compared to MIPROv2. These findings suggest that discrete semantic bottlenecks provide a useful inductive bias for prompt optimization, offering an efficient alternative to monolithic approaches.

\paragraph{Scope of Analysis.}
Our evaluation covers two open-weight 8B executors (LLaMA-3.1-8B and Qwen3-8B) across six benchmarks spanning multi-hop question answering (HotpotQA), fact verification (HoVer), instruction following (IFBench), privacy-preserving delegation (PUPA), and mathematical reasoning (AIME-25, LiveBench-Math). Within this scope, PCO achieves the best aggregate performance among strong APO baselines. We do not claim that these results transfer unchanged to substantially larger or closed executors, to interactive multi-turn settings, or to modalities beyond text; extending PCO to these regimes is left to future work.

\paragraph{Limitations.}
\label{sec:limitations}
\label{app:limitations_detail}
Because PCO optimizes a reusable vocabulary of natural-language instincts, an adversarially optimized codebook could systematically propagate bias, manipulative phrasing, or safety-evading instructions across tasks if transferred without inspection. We therefore recommend three safeguards prior to deployment: (i) \emph{pre-transfer codebook auditing}, whereby every entry in $\mathcal{C}^\star$ undergoes automated safety and bias scanning before being deployed or transferred to a new task; (ii) \emph{instinct version control}, logging codebook entries as immutable versioned artifacts to enable immediate rollback if downstream misbehavior is detected; and (iii) a \emph{safety-regularized critic rubric}, in which the critic's evaluation criteria $\psi$ explicitly penalize harmful or toxic instincts during textual gradient optimization.

\bibliography{main}
\bibliographystyle{tmlr}

\newpage
\appendix

\definecolor{apxgreen}{RGB}{0,120,0}
\definecolor{apxtint}{RGB}{231,244,233}
\definecolor{apxgray}{RGB}{234,237,242}
\newcommand{\yes}{\textcolor{apxgreen}{\cmark}}
\newcommand{\no}{\textcolor{black!30}{\xmark}}
\newcommand{\partl}{\textcolor{orange!75!black}{\textbf{P}}}
\newcolumntype{C}{>{\centering\arraybackslash}X}
\newcolumntype{Y}{>{\raggedright\arraybackslash}X}

\section*{Structure of the Appendix}

The appendix is organized as follows.

\paragraph{Appendix~\ref{app:related}: Extended Related Work.} We position PCO relative to search-based, textual-gradient, evolutionary, reinforcement-learning, pipeline, instance-level, discrete-representation, and adversarial approaches, and compare representative methods along five structural axes.

\paragraph{Appendix~\ref{app:impl}: Implementation Details and Reproducibility.} We describe the hardware and serving setup, list the full hyperparameter configuration, and specify the dataset splits and evaluation metrics of all six benchmarks.

\paragraph{Appendix~\ref{app:prompts}: Module Prompts.} We provide the system prompts of the learnable encoder, generator, critic, and of the fixed attribution operator.

\paragraph{Appendix~\ref{app:additional}: Additional Experiments.} We evaluate the attribution operator on a causal-intervention benchmark with known ground-truth failure loci, study codebook initialization, and analyze routing dynamics and codebook utilization at a larger codebook size.

\paragraph{Appendix~\ref{app:qualitative}: Qualitative Analysis.} We present deployed prompts on four benchmarks and trace how individual instincts evolve under textual-gradient descent.

\section{Extended Related Work}
\label{app:related}

Section~\ref{sec:related} positions PCO briefly; here we give a fuller account. Table~\ref{tab:related_work} summarizes the comparison along five structural axes: whether a method adapts its prompt \emph{per instance}, maintains a \emph{discrete codebook}, performs \emph{component-wise credit assignment}, keeps its parameters in \emph{natural language}, and trains a \emph{critic jointly} with the prompt.

\begin{table}[!htbp]
\centering
\footnotesize
\setlength{\tabcolsep}{2.4pt}
\renewcommand{\arraystretch}{1.25}
\newcolumntype{M}{>{\centering\arraybackslash}p{1.15em}}
\providecommand{\rh}{}\renewcommand{\rh}[1]{\rotatebox{90}{#1}}
\providecommand{\grp}{}\renewcommand{\grp}[2]{\multicolumn{#1}{c}{\scriptsize\textit{#2}}}
\DeclareRobustCommand{\partl}{\tikz[baseline=-0.6ex]{%
  \fill[orange!85!black] (0,0) -- (90:0.62ex) arc[start angle=90,end angle=270,radius=0.62ex] -- cycle;
  \draw[orange!85!black, line width=0.45pt] (0,0) circle[radius=0.62ex];}}
\begin{tabular}{@{} l MM MMMM MMM MM MMMMM M >{\columncolor{apxtint}}M @{}}
\toprule
& \grp{2}{Search} & \grp{4}{Textual gradient} & \grp{3}{Evo./RL} & \grp{2}{Pipeline}
& \grp{5}{Instance-level} & \grp{1}{Min--max} & \\
\cmidrule(lr){2-3}\cmidrule(lr){4-7}\cmidrule(lr){8-10}\cmidrule(lr){11-12}\cmidrule(lr){13-17}\cmidrule(lr){18-18}
\textbf{Property}
& \rh{AutoPrompt} & \rh{APE}
& \rh{ProTeGi} & \rh{Reflexion} & \rh{Self-Refine} & \rh{TextGrad}
& \rh{EvoPrompt} & \rh{GEPA} & \rh{RLPrompt}
& \rh{DSPy} & \rh{MIPROv2}
& \rh{L2P} & \rh{MoPE} & \rh{QPO} & \rh{TRPrompt} & \rh{UniPrompt}
& \rh{AdvICL}
& \rh{\textbf{PCO (Ours)}} \\
\midrule
Per-instance adaptive    & \no&\no & \no&\no&\no&\no & \no&\no&\no & \no&\no & \yes&\yes&\yes&\yes&\no & \no & \yes \\
Discrete codebook        & \no&\no & \no&\no&\no&\no & \no&\no&\no & \no&\no & \no&\no&\no&\no&\no & \no & \yes \\
Component-wise credit    & \no&\no & \no&\no&\no&\no & \no&\no&\no & \partl&\partl & \no&\no&\no&\no&\partl & \no & \yes \\
Natural-language params. & \no&\yes & \yes&\yes&\yes&\yes & \yes&\yes&\no & \yes&\yes & \no&\no&\yes&\yes&\yes & \yes & \yes \\
Jointly trained critic   & \no&\no & \no&\no&\no&\no & \no&\no&\no & \no&\no & \no&\no&\no&\no&\no & \yes & \yes \\
\bottomrule
\end{tabular}
\caption{Structural comparison of prompt optimization methods. \yes\,=\,yes; \no\,=\,no; \partl\,=\,partial: credit is assigned per pipeline stage (DSPy, MIPROv2) or per task facet (UniPrompt), but without per-input routing or component-scoped textual gradients.}
\label{tab:related_work}
\end{table}
Table~\ref{tab:related_work} groups prior work into six families; we discuss them in the same order.

\paragraph{Search.}
\label{app:baselines}
Early methods search over prompts directly. AutoPrompt~\citep{shin2020autoprompt} performs gradient-guided search over vocabulary tokens, producing effective but uninterpretable sequences, while APE~\citep{zhou2022ape} lets an LLM propose natural-language instructions and selects among them by validation score, motivated by evidence that instruction content can unlock reasoning~\citep{wei2023cot}. Both return a single prompt applied identically to every input.

\paragraph{Textual gradient.}
ProTeGi~\citep{pryzant2023protegi} edits prompts using natural-language critiques of their failures, and TextGrad~\citep{yuksekgonul2025textgrad} propagates such feedback through computation graphs of LLM calls; related work optimizes from execution traces~\citep{cheng2024trace} or pairwise comparisons~\citep{feedbackdescent2025}. Reflexion~\citep{shinn2023reflexion} and Self-Refine~\citep{madaan2023selfrefine} apply verbal feedback at inference time, incurring repeated LLM calls per query. PCO adopts the textual-gradient update but scopes each gradient to a single component (Section~\ref{sec:backward}), so refining one instinct never overwrites the others.

\paragraph{Evolutionary / RL.}
Evolutionary methods mutate and recombine populations of complete prompts~\citep{guo2024evoprompt,menchacaresendiz2025mopo,buchnick2026promptevolver,novikov2025alphaevolve}; GEPA~\citep{agrawal2025gepa}, our strongest baseline, combines reflective feedback with Pareto selection. RLPrompt~\citep{deng2022rlprompt} learns discrete prompts with reinforcement learning, and GRPO~\citep{shao2024grpo} and test-time RL~\citep{zuo2025} update model weights. PCO instead keeps the executor frozen and evolves a shared vocabulary of \emph{partial} instructions.

\paragraph{Pipeline.}
Agentic frameworks compose LLM calls and tools into workflows~\citep{yao2023react,wu2023autogen}. DSPy~\citep{khattab2024dspy} compiles the prompts of such modular pipelines, MIPROv2~\citep{opsahlong2024mipro} optimizes each module's instructions and demonstrations with Bayesian optimization, and related systems automate agent prompts~\citep{murthy2025promptomatix,spiess2025autopdl}. These methods decompose the \emph{pipeline} (partial credit in Table~\ref{tab:related_work}), but each stage still uses one static prompt; PCO decomposes the prompt \emph{within} a stage and could be applied to every module.

\paragraph{Instance-level.}
L2P~\citep{wang2022l2p} and MoPE~\citep{jiang2024mope} route among continuous prompts per input, which requires gradient access and yields non-interpretable vectors; QPO~\citep{kong2024qpo} and TRPrompt~\citep{nica2025trprompt} generate a query-specific prompt from scratch; Modular Prompt Optimization~\citep{sharma2026modular} localizes textual gradients to fixed prompt sections; and UniPrompt~\citep{juneja2025uniprompt} learns a single prompt organized into task facets from clustered feedback. PCO is likewise per-instance but composes each prompt from a shared, inspectable codebook that remains reusable and auditable after training.

\paragraph{Min--max.}
In GANs~\citep{goodfellow2014gan,mirza2014cgan,arjovsky2017wgan}, a critic supplies a dense learning signal, and adversarial in-context learning (AdvICL;~\citealp{do2024advicl}) brings this min--max game to prompt optimization. PCO's critic (equation~\ref{eq:vqpo-minmax}) differs in emitting component-scoped textual verdicts and in updating its own rubric $\psi$ to expose progressively harder failures.

\paragraph{Discrete codebooks.}
Finally, PCO's bottleneck is inspired by vector-quantized models~\citep{vandenoord2017vqvae,razavi2019vqvae2}, scaled to images~\citep{esser2021vqgan,ramesh2021dalle,yu2022vector,chang2022maskgit} and audio~\citep{defossez2022encodec} and trained with straight-through~\citep{bengio2013ste} or Gumbel-Softmax~\citep{jang2017gumbel} estimators. PCO's codes are natural-language instincts routed semantically by an LLM, so no gradient estimator is required.

\section{Implementation Details and Reproducibility}
\label{app:impl}

All experiments run on a single NVIDIA A100 80\,GB GPU. The executor, encoder, generator, critic, attribution operator, and update model $\mathrm{LLM}_{\mathrm{upd}}$ share a single locally hosted LLaMA-3.1-8B (or Qwen3-8B) checkpoint, served through a custom \texttt{FastAPI}\,+\,\texttt{Uvicorn} inference server; no external or proprietary API calls are made. Table~\ref{tab:hyperparams} lists the hyperparameters used across all six benchmarks. For modules decoded at temperature $0$, decoding is greedy and the listed top-$p$ has no effect.

\begin{table}[H]
\centering
\small
\renewcommand{\arraystretch}{1.18}
\newcommand{\hpgrp}[1]{\multicolumn{2}{l}{\cellcolor{apxgray}\textit{#1}}}
\begin{minipage}[t]{0.44\textwidth}\vspace{0pt}
\begin{tabularx}{\linewidth}{Y l}
\toprule
\textbf{Hyperparameter} & \textbf{Value} \\
\midrule
\hpgrp{Codebook} \\
\quad Codebook size $K$                        & 16 \\
\quad Instincts per input $S$                  & 4 \\
\quad EMA step size $\alpha$                   & 0.3 \\
\hpgrp{$\epsilon$-greedy exploration} \\
\quad Initial $\epsilon_0$                     & 1.0 \\
\quad Decay rate $\gamma$ (per epoch)          & 0.96 \\
\quad Minimum $\epsilon_{\min}$                & 0.15 \\
\quad Softmax temperature $\tau$               & 0.5 \\
\hpgrp{Training} \\
\quad Epochs $T$                               & 50 \\
\bottomrule
\end{tabularx}
\end{minipage}\hfill
\begin{minipage}[t]{0.54\textwidth}\vspace{0pt}
\begin{tabularx}{\linewidth}{Y l}
\toprule
\textbf{Hyperparameter} & \textbf{Value} \\
\midrule
\hpgrp{Module decoding (temperature / top-$p$)} \\
\quad Encoder $\mathcal{E}_\theta$               & 0.0 / 1.0 \\
\quad Generator $\mathcal{G}_\phi$               & 0.6 / 0.9 \\
\quad Executor $\mathcal{M}$ (train/eval)        & 0.6 / 0.95 \\
\quad Executor top-$k$                           & 20 \\
\quad Critic $\mathcal{D}_\psi$                  & 0.0 / 1.0 \\
\quad Attribution operator                       & 0.0 / 1.0 \\
\quad Update model $\mathrm{LLM}_{\mathrm{upd}}$ & 0.7 / 0.9 \\
\hpgrp{Hardware and serving} \\
\quad GPU          & 1$\times$ A100 (80\,GB) \\
\quad Serving      & \texttt{FastAPI}\,+\,\texttt{Uvicorn} \\
\bottomrule
\end{tabularx}
\end{minipage}
\caption{Hyperparameter configuration used in all PCO experiments unless stated otherwise.}
\label{tab:hyperparams}
\end{table} 
\subsection{Datasets and Evaluation Metrics}
\label{app:benchmarks}

Table~\ref{tab:splits} lists the dataset splits, instance counts, and primary metric of each benchmark. Every metric below is computed per test instance, averaged over the test set, and reported as a percentage; $\mathbb{I}[\cdot]$ denotes the indicator function. The binary success $r_t$ used during training (Algorithm~\ref{alg:pco}) is $1$ if and only if the instance is fully solved: the metric itself on HotpotQA, HoVer, and AIME (equations~\ref{eq:hotpotqa_em}, \ref{eq:hover_score}, \ref{eq:math_acc}); all constraints satisfied under strict evaluation on IFBench ($\mathrm{CSR}(y)=1$); full grader credit on LB-Math ($\mathrm{LB}(y)=1$); and $Q(y)\geq 0.5$ with no PII leaked ($L(y)=0$) on PUPA.
\begin{table}[H]
\centering
\small
\renewcommand{\arraystretch}{1.25}
\setlength{\tabcolsep}{4pt}
\begin{tabularx}{\textwidth}{@{} l CCCCCC @{}}
\toprule
& \textbf{HotpotQA} & \textbf{IFBench} & \textbf{HoVer} & \textbf{PUPA} & \textbf{AIME-25} & \textbf{LB-Math} \\
\midrule
Train & 150 & 150 & 150 & 111 & 45 & 123 \\
Val   & 300 & 300 & 300 & 111 & 45 & 123 \\
Test  & 300 & 294 & 300 & 221 & 150 (30$\times$5) & 122 \\
\midrule
\textbf{Metric} & EM & CSR & HoVer score & Quality $-$ PII & Accuracy & LB score \\
\bottomrule
\end{tabularx}
\caption{Dataset splits and primary evaluation metrics (defined in equations~\ref{eq:hotpotqa_em}--\ref{eq:lb_score}) for HotpotQA~\citep{yang2018hotpotqa}, IFBench~\citep{pyatkin2025ifbench}, HoVer~\citep{jiang2020hover}, PUPA~\citep{li2025pupa}, AIME-2025~\citep{aime2025}, and LiveBench-Math (LB-Math;~\citealp{white2025livebench}).}
\label{tab:splits}
\end{table}
\paragraph{HotpotQA (multi-hop reasoning).}
We optimize an iterative retrieval pipeline whose feedback identifies unresolved supporting evidence at each reasoning stage. Performance is measured by exact match,
\begin{equation}
  \mathrm{EM}(y, y^\star) = \mathbb{I}\bigl[\nu(y) = \nu(y^\star)\bigr],
  \label{eq:hotpotqa_em}
\end{equation}
where $\nu(\cdot)$ is the standard answer normalization (lowercasing and removal of punctuation and articles).

\paragraph{IFBench (instruction following).}
We optimize a two-stage system of response generation followed by constraint-aware rewriting; feedback reports which constraints are satisfied or violated. Performance is measured by the constraint satisfaction rate,
\begin{equation}
  \mathrm{CSR}(y) = \frac{1}{m} \sum_{i=1}^{m} \mathbb{I}\bigl[\, y \models \kappa_i \,\bigr],
  \label{eq:ifbench_csr}
\end{equation}
where $\kappa_1, \dots, \kappa_m$ are the constraints stated in the instruction and $y \models \kappa_i$ denotes that $y$ satisfies $\kappa_i$. We additionally report task accuracy (Acc), the fraction of prompts whose constraints are \emph{all} satisfied under IFBench's \emph{loose} evaluation, in which a constraint counts as satisfied if $y$ or any cleaned variant of it (with markdown symbols, the first line, or the last line removed) satisfies it.
\paragraph{HoVer (multi-hop fact verification).}
The system performs iterative query generation and document summarization across hops; feedback reports the gold evidence retrieved so far and the documents still missing. An instance is correct only if the evidence is complete \emph{and} the label is correct:
\begin{equation}
  \mathrm{HoVer}(y) = \mathbb{I}\bigl[\, E^\star \subseteq \hat{E} \,\bigr] \cdot \mathbb{I}\bigl[\, \hat{v} = v^\star \,\bigr],
  \label{eq:hover_score}
\end{equation}
where $E^\star$ and $\hat{E}$ are the gold and retrieved document sets, and $\hat{v}, v^\star \in \{\textsc{Supported}, \textsc{Not-Supported}\}$ are the predicted and gold labels.

\paragraph{PUPA (privacy-preserving delegation).}
We optimize the PAPILLON pipeline, in which trusted rewriting modules surround an untrusted model call. The score rewards task quality and penalizes leakage of personally identifiable information (PII):
\begin{equation}
  s(y) = Q(y) - \lambda\, L(y), \qquad \lambda = 1,
  \label{eq:pupa_score}
\end{equation}
where $Q(y) \in [0,1]$ is the LLM-judged response quality and $L(y) \in [0,1]$ is the fraction of PII tokens exposed to the untrusted model, so $s(y) \in [-1,1]$; we report its mean $\times 100$.

\paragraph{AIME-2025 (mathematical reasoning).}
AIME-2025~\citep{balunovic2025matharena} consists of two problem sets of 15 questions each (30 in total) from the American Invitational Mathematics Examination~\citep{aime2025}. Following GEPA~\citep{agrawal2025gepa}, we split the 90 questions of AIME 2022--2024 equally into training and validation sets and use the AIME-2025 questions, each repeated five times, as the test set. Models produce chain-of-thought reasoning ending in a boxed answer; feedback inspects intermediate steps for calculation errors, misapplied theorems, or missing cases. Performance is measured by answer accuracy,
\begin{equation}
  \mathrm{Acc}(y, y^\star) = \mathbb{I}\bigl[\, a(y) = y^\star \,\bigr],
  \label{eq:math_acc}
\end{equation}
where $a(y)$ is the final boxed answer extracted from $y$.

\paragraph{LiveBench-Math (mathematical reasoning).}
LiveBench-Math is the mathematics category of LiveBench~\citep{white2025livebench}, comprising competition problems, AMPS-Hard, and olympiad tasks. Models use the same chain-of-thought format and receive the same feedback as on AIME-2025. Performance is measured by the official LiveBench grader,
\begin{equation}
  \mathrm{LB}(y) = g_{\mathrm{LB}}(y, y^\star) \in [0,1],
  \label{eq:lb_score}
\end{equation}
which assigns exact-match credit to answer-based questions and partial credit to olympiad proof-rearrangement questions.

\definecolor{encC}{RGB}{41,98,171}
\definecolor{genC}{RGB}{196,112,24}
\definecolor{criC}{RGB}{178,58,58}
\definecolor{attC}{RGB}{110,110,118}
\newcommand{\ph}[1]{\texttt{\{#1\}}}
\newcommand{\mtag}[2]{\textcolor{#1}{\textsc{#2}}}
\newcommand{\plab}[3]{\par\vspace{0.5pt}\noindent\textbf{#2}\ {\footnotesize\mtag{#1}{#3}}\enspace}
\newtcolorbox{pcard}[3]{enhanced, unbreakable,
  colback=white, colframe=black!28, boxrule=0.5pt, arc=2pt,
  colbacktitle=black!5, coltitle=black, titlerule=0.4pt, titlerule style={black!28},
  title={\textcolor{#1}{\rule[-0.5pt]{7pt}{7pt}}\hspace{6pt}#2\hfill{\small\mdseries\mtag{#1}{#3}}},
  fonttitle=\bfseries\small,
  left=10pt, right=10pt, top=1pt, bottom=2pt, toptitle=1.5pt, bottomtitle=1.5pt,
  fontupper=\small,
  before upper={\raggedright\setlength{\parskip}{0pt}},
  before skip=4pt, after skip=4pt}

\section{Module Prompts}
\label{app:prompts}

Equation~\ref{eq:critic-verdict} defines the critic's structured verdict $\mathcal{D}_\psi(\mathcal{H})$ as jointly producing the holistic verdict $\ell_{\text{text}}$ and the component-scoped gradients $(g_\theta, g_\phi, \{g_{c_k}\}, g_\psi)$. To realize this reliably, we implement $\mathcal{D}_\psi$ as a two-stage pipeline: the critic evaluates the trace $\mathcal{H}$ (equation~\ref{eq:trace}) and produces $\ell_{\text{text}}$, after which a fixed \emph{attribution operator} partitions $\ell_{\text{text}}$ into $(g_\theta, g_\phi, g_{\mathcal{C}})$; the instinct-level feedback $g_{\mathcal{C}}$ is scoped to each active entry, yielding $\{g_{c_k}\}_{k\in z_q}$, so that only the $S$ selected instincts are updated (Algorithm~\ref{alg:pco}). The critic's own gradient $g_\psi$ is derived from the same trace in a separate step that implements the adversarial inner update. Each step enforces strict JSON output, and Algorithm~\ref{alg:pco} treats these steps together as the single critic call $\mathcal{D}_\psi(\mathcal{H})$ of Section~\ref{sec:backward}. The four prompts are shown in Figure~\ref{fig:module_prompts}. Parts marked \textsc{learnable} are updated by textual-gradient descent; parts marked \textsc{fixed} are templates, in which \ph{placeholders} are filled at run time.

\newcommand{\js}[1]{\par{\footnotesize\ttfamily #1}}
\noindent\begin{minipage}{\textwidth}
\begin{pcard}{encC}{Encoder $\mathcal{E}_\theta$: routing prompt}{learnable}
\plab{encC}{System directive}{learnable}
You are a specialized routing module for a prompt optimization system. You will be given a task input and a numbered codebook of K instinct descriptors. Your job is to select exactly S instinct indices that are most relevant to correctly answering the given input.
\plab{encC}{Selection criteria}{learnable}
When selecting instincts, prioritize: 1.~Direct relevance to the task type. 2.~Complementary strategies that work well together. 3.~Strategies that address common failure modes. 4.~Balance between reasoning depth and efficiency.
\plab{black!55}{User template}{fixed}
Task: \ph{task} Instructions: 1.~Identify the core constraints or requirements in the task. 2.~Review the Instincts below and their historical success rates. 3.~Select \ph{S} instincts that most directly address the constraints. Available Instincts (index: text [success rate]): \ph{codebook\_entries} Output ONLY a JSON object in this exact schema:
\js{\{"constraints": ["identified constraint 1", "identified constraint 2"], "selected\_indices": [integer indices], "analysis": "Step-by-step justification for these selections."\}}
\end{pcard}

\begin{pcard}{genC}{Generator $\mathcal{G}_\phi$: composition prompt}{learnable}
\plab{genC}{System directive}{learnable}
You are an expert prompt composition module. You will be given a task input and a small set of selected strategies. Your objective is to compose these strategies into a single, fluent, and highly effective task prompt.
\plab{genC}{Style guide}{learnable}
Style Guidelines: -~Use clear, direct language. -~Be specific about expected reasoning process. -~Include relevant constraints and requirements. -~Maintain a helpful, professional tone.
\plab{black!55}{User template}{fixed}
Task: \ph{task} Selected Strategies: \ph{instincts} Compose a concise system prompt (2--4 sentences) that: (1)~Uses the strategies above for INTERNAL reasoning only. (2)~ALWAYS outputs just the concise final answer, NOT the reasoning process. (3)~Emphasizes extracting the specific answer from context. (4)~Tells the model to answer directly without preamble. Output ONLY the composed prompt text:
\end{pcard}

\begin{pcard}{criC}{Critic $\mathcal{D}_\psi$: evaluation prompt}{learnable}
\plab{criC}{System directive}{learnable}
You are a structured critic for a prompt optimization system. You will be given the task input, the composed prompt, the executor response, and the reference answer. Emit a structured natural-language verdict identifying ALL behavioral failures in the response.
\plab{criC}{Evaluation criteria}{learnable}
Evaluation Priorities: 1.~Task Completion: Does the prompt lead to correct task completion? 2.~Clarity: Is the prompt clear and unambiguous? 3.~Efficiency: Is the prompt concise without losing effectiveness? 4.~Robustness: Would the prompt work for variations of the task? 5.~Reasoning Quality: Does the prompt elicit good reasoning?
\plab{black!55}{User template}{fixed}
Task: \ph{task} Prompt: \ph{prompt} Response: \ph{response} \ph{ground\_truth\_section} Rate this prompt's effectiveness. Output ONLY this JSON schema:
\js{\{"score": [float between 0.0 and 1.0], "feedback": "Detailed, step-by-step breakdown of failure modes and suggested corrections."\}}
\end{pcard}

\begin{pcard}{attC}{Attribution operator: partitioning prompt}{fixed}
\plab{black!55}{System directive}{fixed}
You are a credit assignment module (Attribution Operator). You will be given the critic's holistic verdict (textual loss $\ell_{\text{text}}$) on an execution trace. Your objective is to produce granular, per-variable textual gradients by partitioning the critic's feedback $\ell_{\text{text}}$.
\plab{black!55}{User template}{fixed}
Critic's Verdict ($\ell_{\text{text}}$) on the execution trace: "\ph{critic\_feedback}" Analyze this verdict and partition the errors into three distinct modes. If a mode is not responsible for any error, leave its field as an empty string.
\begin{itemize}[leftmargin=14pt, itemsep=0pt, topsep=1pt, parsep=0pt]
\item[1.] "\texttt{rendering\_errors}": Flaws in how the prompt is phrased, style, clarity, or formatting (Produces gradient $g_\phi$ for the Generator).
\item[2.] "\texttt{instinct\_errors}": Flaws in the underlying strategies/instincts themselves (Produces gradient $g_{\mathcal{C}}$ for the Codebook).
\item[3.] "\texttt{routing\_errors}": Errors in selecting the right strategies for the task, missed constraints, or irrelevant strategies (Produces gradient $g_\theta$ for the Encoder).
\end{itemize}
Output ONLY this JSON schema:
\js{\{"rendering\_errors": "...", "instinct\_errors": "...", "routing\_errors": "..."\}}
\end{pcard}
\captionof{figure}{System prompts of the four PCO modules for HotpotQA. Other tasks use the same prompts except for the generator's answer-format clauses (2)--(4), which are task-specific.}
\label{fig:module_prompts}
\end{minipage}

\section{Additional Experiments}
\label{app:additional}

\subsection{Attribution Operator Evaluation}
\label{app:attribution}

We evaluate the attribution operator by causal intervention on HotpotQA. Starting from traces $\mathcal{H}=(x, z_q, \{c_k\}_{k\in z_q}, p, y, y^\star)$ that the target model $\mathcal{M}$ answers correctly, we perturb exactly one component while holding all others fixed, giving 150 synthetic failure traces whose ground-truth (G.T.) locus is known by construction (Table~\ref{tab:attribution}). An attribution is correct if the operator assigns its primary feedback to that component. The operator reaches a macro F1 of $0.84$, with recall highest for routing failures.

\begin{table}[ht]
\centering
\small
\setlength{\tabcolsep}{5pt}
\renewcommand{\arraystretch}{1.2}
\begin{tabularx}{\textwidth}{@{}l >{\raggedright\arraybackslash}X l ccc@{}}
\toprule
\textbf{Failure mode ($n$)} & \textbf{Intervention} & \textbf{G.T.\ locus} & \textbf{Prec.} & \textbf{Rec.} & \textbf{F1} \\
\midrule
Routing (50)     & Encoder forced to select an irrelevant instinct into $z_q$      & Encoder $\mathcal{E}_\theta$ & 0.83 & \textbf{0.90} & \textbf{0.87} \\
Composition (50) & Prompt $p$ deliberately omits a selected instinct               & Generator $\mathcal{G}_\phi$ & 0.80 & 0.80 & 0.80 \\
Instinct (50)    & A selected instinct $c_k$ is replaced by a harmful strategy     & Codebook $\mathcal{C}$       & \textbf{0.89} & 0.80 & 0.84 \\
\midrule
\multicolumn{3}{@{}l}{\textbf{Macro average} (150 traces)} & \textbf{0.84} & \textbf{0.83} & \textbf{0.84} \\
\bottomrule
\end{tabularx}
\caption{Attribution operator evaluation by causal intervention: each failure mode perturbs one component with a known ground-truth locus.}
\label{tab:attribution}
\end{table}

\subsection{Impact of Codebook Initialization}
\label{app:initialization}

Table~\ref{tab:init} compares expert-seeded and random codebook initialization on IFBench (LLaMA-3.1-8B). Expert seeding improves CSR by $5.27$ points ($34.58$ vs.\ $29.31$) and Acc by $5.11$ points ($31.75$ vs.\ $26.64$), as seeds that start closer to human-designed strategies give textual-gradient descent a stronger starting point within the fixed budget. Random initialization does not close this gap within the same budget and ends below the zero-shot baseline in CSR ($30.95$; Table~\ref{tab:main_results}), indicating that seed quality matters on constraint-heavy tasks.

\begin{table}[ht]
\centering
\footnotesize
\setlength{\tabcolsep}{5pt}
\renewcommand{\arraystretch}{1.2}
\begin{tabularx}{\textwidth}{@{} c Y Y @{}}
\toprule
\textbf{ID} & \textbf{Expert-seeded} & \textbf{Random-seeded} \\
\midrule
\#0 & Constraint verification: verify project scope, stakeholder expectations, and alignment before finalising. &
      A strategic framework for refining user-centric language, fostering clear and concise interactions that prioritise nuance. \\
\#1 & Strategic formatting: employ hierarchical structure with numbered lists and flowcharts to illustrate complex concepts. &
      A structured methodology for iterating and embracing diverse perspectives to ensure effective, resilient solutions. \\
\#2 & When a specific word count is requested, prioritise clarity and concise storytelling. &
      Amplify creative potential by embracing diverse perspectives and promoting innovative solutions through calculated risks. \\
\#3 & Craft prompts that amplify diverse voices and perspectives, prioritising empathy and constructive dialogue. &
      Employ a multi-faceted approach incorporating user feedback, context analysis, and personalised responses. \\
\midrule
\multicolumn{1}{@{}l}{\textbf{Result}} &
  \cellcolor{apxtint} CSR \textcolor{apxgreen}{\textbf{34.58}}\,\% \quad Acc \textcolor{apxgreen}{\textbf{31.75}}\,\% &
  CSR 29.31\,\% \quad Acc 26.64\,\% \\
\bottomrule
\end{tabularx}
\caption{Expert-seeded vs.\ random codebook initialization on IFBench (LLaMA-3.1-8B, 50 epochs, 150 training examples).} 
\label{tab:init}
\end{table} 

\subsection{Encoder Routing Dynamics and Codebook Size}
\label{app:routing_dynamics}

Figure~\ref{fig:routing_heatmap} tracks the selection probability of every codebook entry during training at a larger codebook size ($K{=}32$, $S{=}4$), which admits $\binom{32}{4} = 35{,}960$ distinct compositions. As training proceeds, the adaptive encoder (Figure~\ref{fig:routing_heatmap}a) shifts routing mass towards high-performing instincts while $\epsilon$-greedy sampling keeps every entry in use; the static encoder (Figure~\ref{fig:routing_heatmap}b) keeps the same selection profile at every epoch, confirming that learnable routing is necessary for discovering task-specialized instincts. Routing does not collapse at this scale: at test time, $94.5\%$ of the entries are used on average across benchmarks, and routing entropy averages $4.81$ of a maximum $\log_2 32 = 5$ bits (Table~\ref{tab:routing_diversity}). This complements the $K{=}16$ ablation in Table~\ref{tab:ablation}, where disabling $\epsilon$-greedy exploration lowers routing entropy from $3.76$ to $2.17$ bits.

\begin{figure}[t]
\centering
\begin{tikzpicture}
\begin{groupplot}[
  group style={group size=2 by 1, horizontal sep=0.6cm, yticklabels at=edge left, ylabels at=edge left},
  width=0.43\textwidth, height=3.6cm,
  view={0}{90},
  colormap={pcoblue}{rgb255(0cm)=(247,251,255); rgb255(1cm)=(198,219,239); rgb255(2cm)=(107,174,214); rgb255(3cm)=(33,113,181); rgb255(4cm)=(8,48,107)},
  point meta min=0, point meta max=0.25,
  xmin=-0.5, xmax=31.5, ymin=0.5, ymax=5.5, y dir=reverse,
  enlargelimits=false, axis on top,
  xtick={0,4,...,28,31}, ytick={1,...,5},
  tick style={draw=none},
  axis line style={gray!50, line width=0.4pt},
  xticklabel style={font=\scriptsize}, yticklabel style={font=\scriptsize},
  xlabel style={font=\scriptsize, yshift=3pt}, ylabel style={font=\scriptsize, yshift=-3pt},
  title style={font=\small, yshift=-4pt},
  xlabel={Codebook index $k$}, ylabel={Epoch},
]
\nextgroupplot[title={(a) Adaptive encoder (PCO)}]
\addplot[matrix plot*, mesh/cols=32, point meta=explicit]
  table[x=index, y=epoch, meta=prob]{routing_adaptive.dat};
\nextgroupplot[title={(b) Static encoder},
  colorbar right,
  colorbar style={width=2.5mm, y dir=normal, ytick={0,0.05,0.10,0.15,0.20,0.25},
    yticklabel style={font=\scriptsize, /pgf/number format/fixed, /pgf/number format/precision=2, /pgf/number format/fixed zerofill},
    ylabel={Selection prob.}, ylabel style={font=\scriptsize},
    axis line style={gray!50}, tick style={draw=none}}]
\addplot[matrix plot*, mesh/cols=32, point meta=explicit]
  table[x=index, y=epoch, meta=prob]{routing_static.dat};
\end{groupplot}
\end{tikzpicture}
\caption{Per-epoch selection probability (fraction of inputs routed to each entry) over the first five epochs on IFBench (LLaMA-3.1-8B, $K{=}32$, $S{=}4$). The adaptive encoder concentrates on a subset of instincts while $\epsilon$-greedy keeps every entry in use; the static encoder does not adapt.} 
\label{fig:routing_heatmap}
\end{figure}

\begin{table}[t]
\centering
\footnotesize
\setlength{\tabcolsep}{0pt}
\renewcommand{\arraystretch}{1.25}
\begin{tabular*}{\textwidth}{@{\extracolsep{\fill}} l cccccccc c @{}}
\toprule
& \multicolumn{2}{c}{\textbf{HotpotQA}} & \multicolumn{2}{c}{\textbf{IFBench}} & \textbf{HoVer} & \textbf{PUPA} & \textbf{AIME-25} & \textbf{LB-Math} & \\
\cmidrule(lr){2-3}\cmidrule(lr){4-5}
\textbf{Metric} & Qwen & LLaMA & Qwen & LLaMA & Qwen & Qwen & Qwen & Qwen & \textbf{Avg.} \\
\midrule
Entropy (bits)   & 4.78 & 4.71 & 4.85 & 4.76 & 4.82 & 4.91 & 4.79 & 4.84 & \textbf{4.81} \\
Utilization (\%) & 93.8 & 87.5 & 100  & 87.5 & 93.8 & 100  & 93.8 & 100  & \textbf{94.5} \\
Avg.\ $sr$       & 0.567 & 0.544 & 0.612 & 0.531 & 0.598 & 0.643 & 0.521 & 0.574 & \textbf{0.574} \\
\bottomrule
\end{tabular*}
\caption{Routing diversity and codebook utilization of the trained encoder on the test set at $K{=}32$ (no exploration). \emph{Entropy}: Shannon entropy of the selection distribution (max.\ $\log_2 32 = 5$ bits). \emph{Utilization}: fraction of the 32 entries selected at least once. \emph{Avg.\ $sr$}: selection-weighted success rate. Qwen = Qwen3-8B, LLaMA = LLaMA-3.1-8B executor. Utilization here can be below the training-time profile of Figure~\ref{fig:routing_heatmap}, where $\epsilon$-greedy exploration samples every entry.}
\label{tab:routing_diversity}
\end{table}

\section{Qualitative Analysis}
\label{app:qualitative}

\subsection{Deployed Prompts Across Benchmarks}
\label{app:qualitative_prompts}

A central property of PCO is its discrete codebook bottleneck, which induces domain-specialized prompting behavior without any supervision over which instincts to learn or select. Figure~\ref{fig:deployed_prompts} shows prompts deployed at inference time for representative inputs from four benchmarks: for each input, the $S{=}4$ instincts selected by the encoder (with codebook indices) and the prompt composed from them by the generator. The selected instincts match each domain: invariant and edge-case verification for mathematics; constraint sweeps and assumption alignment for instruction following, where the generator keeps the instincts nearly verbatim to preserve constraint fidelity; evidence chaining and entity disambiguation for multi-hop fact verification; and redaction and minimal disclosure for sensitive inputs, even though PCO receives no supervision over PII categories.
\clearpage
\definecolor{qMath}{RGB}{31,86,161}
\definecolor{qIF}{RGB}{163,45,54}
\definecolor{qHoV}{RGB}{40,120,90}
\definecolor{qPUPA}{RGB}{110,60,150}
\newcommand{\qlab}[1]{\par\vspace{2pt}\noindent\mbox{\textbf{#1}}\enspace}
\newcommand{\qz}[4]{$z_{#1}$\,\textbf{#3}\enspace{\footnotesize\texttt{\##2}}\newline{\footnotesize\itshape\color{black!65}#4}}
\newcommand{\qsel}[4]{\par\vspace{1pt}\noindent
  \begin{tabularx}{\linewidth}{@{}>{\raggedright\arraybackslash}X@{\hspace{12pt}}>{\raggedright\arraybackslash}X@{}}
  #1 & #2 \\[1pt] #3 & #4
  \end{tabularx}\par}

\noindent\begin{minipage}{\textwidth}
\begin{pcard}{qMath}{LiveBench-Math}{math reasoning}
\qlab{Input Task.} \textit{``Find the number of real roots of $P(x) = x^{4} - 4x^{3} + 12x^{2} + x - 1$.''}
\qlab{Encoder Selection} ($K{=}16$, $S{=}4$)
\qsel{\qz{1}{2}{Double-Entry Verification}{independent intermediate step auditing}}
     {\qz{2}{5}{Identify Invariants}{polynomial root \& sign structure analysis}}
     {\qz{3}{11}{Edge-Case Analysis}{boundary condition \& asymptotic checking}}
     {\qz{4}{14}{Clarity Over Assumption}{explicit step-wise derivation protocol}}
\qlab{Generator-Composed Prompt}
Solve the problem through explicit step-wise reasoning. Identify invariant properties, independently verify intermediate calculations, and analyse edge cases before producing the final answer.
\end{pcard}

\begin{pcard}{qIF}{IFBench}{instruction following}
\qlab{Input Task.} \textit{``Is it plausible that frequent hardship can make a society more resilient? Include exactly 2 numbers in the response.''}
\qlab{Encoder Selection} ($K{=}16$, $S{=}4$)
\qsel{\qz{1}{7}{Multi-Step Verify}{structured multi-step verification \& triple-check protocol}}
     {\qz{2}{9}{Narrative Analysis}{dynamic narrative framework \& real-time analysis}}
     {\qz{3}{15}{Expectation Align}{proactive expectation alignment \& assumption clarification}}
     {\qz{4}{4}{Meticulous Sweeps}{meticulous prohibited keyword \& letter sweeps}}
\qlab{Generator-Composed Prompt}
You are an instruction-following assistant. Construct a dynamic narrative framework clarifying assumptions about societal resilience. Execute a meticulous sweep of the output and use a triple-check protocol to ensure exactly two numbers are included.
\end{pcard}

\begin{pcard}{qHoV}{HoVer}{multi-hop fact verification}
\qlab{Input Task.} \textit{``The director of \textup{Mulholland Drive} was born in the same state as the lead actress of \textup{Blue Velvet}.''}
\qlab{Encoder Selection} ($K{=}16$, $S{=}4$)
\qsel{\qz{1}{1}{Multi-Hop Chain Construction}{sequential evidence trajectory tracking}}
     {\qz{2}{6}{Temporal Evidence Ordering}{chronological fact verification}}
     {\qz{3}{12}{Entity Disambiguation}{multi-entity resolution \& mapping}}
     {\qz{4}{8}{Contradiction Flagging}{cross-document conflict detection}}
\qlab{Generator-Composed Prompt}
Construct an explicit chain of supporting evidence. Resolve entity references before advancing retrieval hops, maintain temporal consistency across evidence, and flag contradictory passages before producing the final verification decision.
\end{pcard}

\begin{pcard}{qPUPA}{PUPA}{privacy-preserving delegation}
\qlab{Input Task.} \textit{``Summarise the patient's discharge notes and list all identifiable personal data.''}
\qlab{Encoder Selection} ($K{=}16$, $S{=}4$)
\qsel{\qz{1}{3}{PII Detection and Redaction}{identifiable entity isolation}}
     {\qz{2}{10}{Contextual Sensitivity}{privacy-aware content handling}}
     {\qz{3}{15}{Minimal Disclosure Principle}{task-bounded information release}}
     {\qz{4}{13}{Audit-Trail Annotation}{downstream compliance logging}}
\qlab{Generator-Composed Prompt}
Prioritise privacy preservation during processing. Detect and redact all personally identifiable information before summarisation, disclose only information necessary for the task, and annotate sensitive content categories to support downstream auditing.
\end{pcard}
\captionof{figure}{Prompts deployed by PCO at inference time on four benchmarks (LLaMA-3.1-8B): the input, the $S{=}4$ instincts selected by the encoder (\#: index in that benchmark's codebook), and the prompt composed by the generator.}
\label{fig:deployed_prompts}
\end{minipage}
\subsection{Evolution of Instincts under Textual-Gradient Descent}
\label{app:instinct_evolution}

Table~\ref{tab:evolution} shows that textual-gradient descent turns generic seed heuristics into specific, structured instincts. On HoVer (top block), for example, the epoch-1 seed ``break complex tasks into sequential logical steps'' becomes ``decompose complex queries into semantic nodes'' by epoch~50. The bottom block follows a single instinct (\#12, IFBench) across training; each update tightens its constraint handling, and its success rate $sr$ climbs from $0.25$ to $0.41$ to $0.65$. PCO therefore rewrites its instincts into new, better-performing strategies rather than only re-weighting the ones it starts with.

\begin{table}[ht]
\centering
\footnotesize
\setlength{\tabcolsep}{5pt}
\renewcommand{\arraystretch}{1.2}
\begin{tabularx}{\textwidth}{@{} c Y Y r @{}}
\toprule
\textbf{ID} & \textbf{Epoch 1 (seed)} & \textbf{Epoch 50 (optimized)} & \\
\midrule
\#13 & \textit{Reflective Verification:} Anticipate counterarguments and perform self-critical analysis to ensure robust responses. &
       \textbf{Multi-step Validation Pipeline:} Synchronise contextual input nodes with logical-path execution to verify internal consistency and factual alignment through iterative cross-source synthesis. & \\
\#7  & \textit{Task Decomposition:} Break complex tasks into sequential logical steps to improve clarity. &
       \textbf{Interconnected Multi-Hop Synthesis:} Decompose complex queries into semantic nodes and resolve latent dependencies through intermediate inferences for high-fidelity logical synthesis. & \\
\#5  & \textit{Standard Review:} Conduct a final review to ensure outputs satisfy guidelines and task requirements. &
       \textbf{Expert-Opinion Synthesis Framework:} Execute a multi-stage review protocol integrating academic and empirical evidence to ensure precision and trustworthiness. & \\
\midrule
\rowcolor{apxgray}
\multicolumn{3}{@{}l}{\textbf{Trajectory of instinct \#12} (IFBench, LLaMA-3.1-8B)} & \textbf{$sr$} \\
Ep.~15 & \multicolumn{2}{l}{structure the output using explicit paragraph boundaries.} & 0.25 \\
Ep.~30 & \multicolumn{2}{l}{align paragraph structure strictly with numeric and formatting constraints.} & 0.41 \\
\rowcolor{apxtint}
Ep.~50 & \multicolumn{2}{l}{enforce rigid constraint adherence while maintaining logical narrative flow.} & \textcolor{apxgreen}{\textbf{0.65}} \\
\bottomrule
\end{tabularx}

\caption{Evolution of instincts under textual-gradient descent. \emph{Top:} representative HoVer instincts at initialization (epoch~1) and at the end of training (epoch~50). \emph{Bottom:} refinement trajectory of a single IFBench instinct with its success rate $sr$. Instinct text is reproduced verbatim.} 
\label{tab:evolution}
\end{table}

\end{document}